\documentclass[english,10pt]{article}
\usepackage[papersize={19cm,25cm},body={15cm,21cm},centering]{geometry}
\usepackage[T1]{fontenc}
\usepackage[latin9]{inputenc}
\usepackage{amsmath}
\usepackage{amssymb}
\usepackage{url}
\usepackage{babel}
 \usepackage{graphicx}
\usepackage{epsfig,subfigure,color}
\usepackage{graphics}
 \usepackage{graphicx}
\usepackage{amsmath,amsfonts,amsthm,latexsym}
\usepackage{mathrsfs,amsbsy}
\usepackage{multirow}
\usepackage{epstopdf}
\usepackage{algorithm}
\usepackage{algorithmic}
\usepackage{cite}

\newtheorem{lemma}{Lemma}
\newtheorem{property}{Property}
\newtheorem{example}{Example}
\newtheorem{remark}{Remark}

\theoremstyle{definition}
\newtheorem{defn}{Definition}

\DeclareMathOperator{\rank}{rank}
\DeclareMathOperator{\tr}{trace}
\DeclareMathOperator{\diag}{diag}
\DeclareMathOperator{\argmax}{argmax}

\begin{document}
\title{Generalized Separable Nonnegative Matrix Factorization}
\date{}
\author{
Junjun Pan \qquad  Nicolas Gillis\thanks{
Emails: \{junjun.pan, nicolas.gillis\}@umons.ac.be.
This work was supported by the European Research Council (ERC starting grant no 679515), and 	
				the Fonds de la Recherche Scientifique - FNRS and the Fonds Wetenschappelijk Onderzoek - Vlanderen (FWO) under EOS Project no O005318F-RG47. } \\
Department of Mathematics and Operational Research \\
Facult\'e Polytechnique, Universit\'e de Mons \\
Rue de Houdain 9, 7000 Mons, Belgium
}

\maketitle
\begin{abstract}
Nonnegative matrix factorization (NMF) is a linear dimensionality technique for nonnegative data with applications such as image analysis, text mining, audio source separation and hyperspectral unmixing. Given a data matrix $M$ and a factorization rank $r$, NMF looks for a nonnegative matrix $W$ with $r$ columns and a nonnegative matrix $H$ with $r$ rows such that $M \approx WH$. NMF is NP-hard to solve in general. However, it can be computed efficiently under the separability assumption which requires that the basis vectors appear as data points, that is, that there exists an index set $\mathcal{K}$ such that $W = M(:,\mathcal{K})$. In this paper, we generalize the separability assumption: We only require that for each rank-one factor $W(:,k)H(k,:)$ for $k=1,2,\dots,r$, either $W(:,k) = M(:,j)$ for some $j$ or $H(k,:) = M(i,:)$ for some $i$. We refer to the corresponding problem as generalized separable NMF (GS-NMF). We discuss some properties of GS-NMF and propose a convex optimization model which we solve using a fast gradient method. We also propose a heuristic algorithm inspired by the successive projection algorithm.  To verify the effectiveness of our methods, we compare them with several state-of-the-art separable NMF algorithms on synthetic, document and image data sets.
\end{abstract}


\section{Introduction}
Given a nonnegative matrix $M \in \mathbb{R}^{m\times n}_+$ and an integer factorization rank $r$,
nonnegative matrix factorization (NMF) is the problem of computing
$W \in \mathbb{R}^{m\times r}_+$ and
$H \in \mathbb{R}^{r\times n}_+$
such that $M \approx WH$.
Typically, the columns of the input matrix $M$ correspond to data points (such as images of pixel intensities or documents of word counts) and NMF allows to perform linear dimensionality reduction.
In fact, we have $M(:,j) \approx \sum_{k=1}^r W(:,k) H(k,j)$ for all $j$, where $M(:,j)$ denotes the $j$th column of $M$.
This means that the data points are approximated by points within an $r$-dimensional subspace spanned by the columns of $W$. The nonnegativity constraints lead to easily interpretable factors with applications such as image processing, text mining, hyperspectral unmixing and audio source separation; see for example the recent survey~\cite{xiao2019uniq} and the references therein.

NMF is NP-hard in general~\cite{vavasis2009complexity} and its solution is in most cases not unique; see~\cite{xiao2019uniq} and the references therein.
These two issues motivated the introduction of the separability assumption as a way to solve NMF efficiently and have unique solutions.
A matrix $M\in \mathbb{R}^{m\times n}$ is $r$-separable if there exist an index set $\mathcal{K}$ of cardinality $r$ and a nonnegative matrix $H$ such that $M=M(:,\mathcal{K})H$,  where $M(:,\mathcal{K})$ is the matrix containing the columns of $M$ with index in $\mathcal{K}$.
This means that there exists an NMF $(W,H)$ such that each column of $W$ is equal to a column of $M$.
Given a matrix $M$ that satisfies the separability condition,
computing $W=M(:,\mathcal{K})$ and $H$ can be done efficiently; see for example~\cite{ma2014signal, gillis2014and} and the references therein.  The corresponding problem is referred to as separable NMF.

Let us present an equivalent definition of separability that will be particularly useful in this paper. It was originally proposed in~\cite{esser2012convex, recht2012factoring, elhamifar2012see} in order to design convex formulations for separable NMF.
The matrix $M$ is $r$-separable if there exist some permutation matrix $\Pi\in \{0,1\}^{n\times n}$ and a nonnegative matrix $H'\in \mathbb{R}^{r\times (n-r)}_+$ such that
\[
M\Pi=M\Pi\left(
           \begin{array}{cc}
             I_r  & H'  \\
              0_{n-r,r} & 0_{n-r,n-r}  \\
           \end{array}
         \right),
\]
where $I_r$ is the $r$-by-$r$ identity matrix and $0_{r,p}$ is the matrix of all zeros of dimension $r$ by $p$.
In fact, under an appropriate permutation, the first $r$ columns of $M$ correspond to the columns of $W$ while the last $n-r$ columns are convex combinations of these first $r$ columns. Equivalently, we have
\begin{equation} \label{eq:convsep}
M \; = \; M \; \underbrace{ \Pi \left(
           \begin{array}{cc}
             I_r  & H'  \\
              0_{n-r,r} & 0_{n-r,n-r}  \\
           \end{array}
         \right) \Pi^{T} }_{X \in \mathbb{R}^{n \times n}}.
\end{equation}
Convex formulations were obtained by trying to find a matrix $X$ such that (i)~$M \approx MX$ and (ii)~$X$ has as many zero rows as possible~\cite{esser2012convex, recht2012factoring, elhamifar2012see, gillis2014robust}; see Section~\ref{sec:optim} for more details.

Note that every $m$-by-$n$ nonnegative matrix is $n$-separable since $M = MI_n$ hence it is important to find the minimal $r$. Geometrically, in noiseless conditions, the minimal $r$ is the number of extreme rays of the cone generated by the columns of $M$.

Under the separability assumption, NMF can be solved in polynomial time. This has been known and used for a long time in the hyperspectral imaging and signal processing communities~\cite{ren2003automatic,nascimento2005vertex, chan2008convex, chan2011simplex};see also~\cite{ma2014signal} and the references therein. Furthermore, separable NMF can still be solved in polynomial time in the presence of noise~\cite{arora2012computing, arora2016computing}, and many robust algorithms have been proposed recently \cite{recht2012factoring, arora2013practical, gillis2013robustness, gillis2014fast, gillis2014successive}.
The separability assumption makes sense in several practical situations.
In hyperspectral unmixing, each column of the data matrix is the spectral signature of a pixel. Separability requires that for each material wtihin the hyperspectral image, there exists a pixel that contains only that material; see for example~\cite{ma2014signal}.
In audio source separation, the input matrix is the time  frequency amplitude spectrogram~\cite{fevotte2009nonnegative}. Separability requires that for each source, there exists a moment in time when only that source is active (or, considering the matrix transpose, separability requires that, for each source, there is a frequency for which only that source is active).

In document classification, each entry $M(i,j)$ of matrix $M$ indicates the importance of word $i$ in document $j$ (e.g., the number of occurrences of word~$i$ in document~$j$). 
Separability of $M$ (that is, each column of $W$ appears as a column of $M$) requires that, for each topic, there exists at least one document only discuss that topic (a ``pure'' document).
Separability of $M^T$ (that is, each row of $H$ appears as a row of $M$) requires that, for each topic,  there exists at least one word used only in that topic\cite{arora2012learning} (a ``pure'' word, referred to as an anchor word).

In this paper, we generalize the separability assumption as follows.
\begin{defn} \label{def:gsmat1}
A matrix $M\in \mathbb{R}_+^{m\times n}$
is $(r_1,r_2)$-separable if there exist
    an index set $\mathcal{K}_1$ of cardinality $r_1$
    and
    an index set $\mathcal{K}_2$ of cardinality $r_2$, and nonnegative matrices $P_1\in \mathbb{R}_+^{r_1\times n}$ and $P_2\in \mathbb{R}_+^{m\times r_2}$ such that
\begin{equation}\label{Gsep}
 M=M(:,\mathcal{K}_1)P_1+P_2M(\mathcal{K}_2,:),
\end{equation}
where $P_1(:,\mathcal{K}_1) = I_{r_1}$
and
$P_1(\mathcal{K}_2,:) = I_{r_2}$.
\end{defn}
We will refer to such matrices as generalized separable (GS) matrices, with their corresponding GS decomposition~\eqref{Gsep}.

The $(r_1,r_2)$-separability is a natural extension of  $r$-separability since a matrix is $r$-separable if and only if it is $(r,0)$-separable.
Note that every $m$-by-$n$ nonnegative matrix $M$ is $(n,0)$- and $(0,m)$-separable.

Generalized separability makes sense in several applications.
In document classification, it requires that for each topic, there exists either
\begin{itemize}
  \item a document discussing only that topic (a ``pure''  document), \emph{or}
  \item a word used only by that topic (an anchor word).
\end{itemize}

This is a much more relaxed condition than separability which needs a ``pure''  document for each topic, or an anchor word for each topic when considering the matrix transpose.

In a GS decomposition, the $r_1$ columns of $M$ indexed by $\mathcal{K}_1$ are $r_1$ pure documents, and the
$r_2$ rows of $M$ indexed by $\mathcal{K}_2$ are $r_2$ anchor words, for a total of $r = r_1 + r_2$ topics.
The conditions $P_1(:,\mathcal{K}_1) = I_{r_1}$
and
$P_1(\mathcal{K}_2,:) = I_{r_2}$ are rather natural and mean that the pure documents and anchor words are represented by themselves.

In terms of audio source separation~\cite{fevotte2009nonnegative}, generalized separability requires that for each source there exists either
\begin{itemize}
  \item a moment in time where only that source is active, or
  \item a frequency where only that source has a positive signature.
\end{itemize}

Again, this condition is much more relaxed than separability. 

 \subsection{Related problems}
 GS-NMF is related to the CUR decomposition and the
 pseudo-skeleton approximation.
 Given a matrix $M$, these techniques try to identify a subset of columns $\mathcal{K}_1$ and rows $\mathcal{K}_2$ of $M$ such that
$||M -  M(:,\mathcal{K}_1) U M(\mathcal{K}_2,:)||_F$
 is as small as possible.
 In the CUR decomposition, $U$ is chosen so as to minimize the approximation error, that is, $U = M(:,\mathcal{K}_1)^{\dagger} M M(\mathcal{K}_2,:)^{\dagger}$, where $A^{\dagger}$ denotes a Moore-Penrose generalized inverse of the matrix $A$~\cite{mahoney2009cur}.
 In the skeleton approximation, $U = M(\mathcal{K}_2,\mathcal{K}_1)^{-1}$~\cite{goreinov1997pseudo}. We refer the reader to~\cite{mikhalev2018rectangular} and the references therein for more information on these models.
Since these models do not take nonnegativity into account, their analysis is rather different than GS-NMF.
For example, to obtain exact decompositions of a rank-$r$ matrix,
they can pick any subset of $r$ linearly independent rows and columns; this is not true for GS-NMF.


GS-NMF is also related to a model introduced in~\cite{liu2011latent} and referred to as latent low-rank representation (LatLRR). The goal of LatLRR is to use the representation $M \approx MX + YM$ where both $X$ and $Y$ have low-rank. Intuitively, $M$ is represented using a subspace of the column space of $M$ (namely, $MX$) and a subspace of the row space (namely, $YM$).
To achieve this goal, Liu and Yan~\cite{liu2011latent}  minimize the nuclear norms of $X$ and $Y$ (since minimizing the rank is hard in general) and apply their model on facial images.  GS-NMF clearly shares some similarity with LatLRR. In fact, our convex model introduced in Section~\ref{sec:optim} will also use the representation $M = MX + YM$ but the constraints on $X$ and $Y$ will be rather different.
Moreover, GS-NMF takes nonnegativity into account, and it is more interpretable as the basis used to reconstruct $M$ cannot be any linear combinations of the columns/rows of $M$ as in LatLRR, but need to be a subset of these columns/rows.
For example, when applied on facial images (see Section~\ref{sec:images} for numerical experiments), our model will identify important pixels and images within a data set (meaning that they can be used to approximate well all other images)
while LatLRR identifies important linear combinations of pixels and images which is more difficult to interpret.

 \subsection{Outline and contribution of the paper}

 In this paper, we consider the NMF problem under the generalized separability condition, referred to as generalized separable NMF (GS-NMF).

 In Section~\ref{sec:prop}, we provide an equivalent characterization of GS matrices, similarly as done for separable matrices in~\eqref{eq:convsep}.
 This leads to an idealized model to tackle GS-NMF.
 We then present several properties of GS matrices. We present a class of $m$-by-$n$ matrices which are not $(n-1,0)$- nor $(0,m-1)$-separable but that are $(3,3)$-separable.
 This illustrates the fact that GS decompositions can be much more compact than separable ones, requiring much fewer rank-one factors to reconstruct the input matrix.
 We also discuss non-uniqueness issues of GS-NMF, a problem which is not present for separable NMF.
 In Section~\ref{sec:optim}, we propose a convex optimization model to tackle GS-NMF. It is the generalization of the models proposed in~\cite{recht2012factoring, gillis2014robust} for separable NMF.     Then, we implement a fast gradient method to takle this model that will allow us to tackle GS-NMF,
 similarly as done in~\cite{gillis2018fast} for separable NMF.
 Unfortunately, this model requires the use of $n^2+m^2$ variables hence is computationally rather expensive and does not scale well for large data sets. 
    In Section~\ref{sec:heur}, we propose a heuristic algorithm inspired by the successive projection algorithm (SPA)~\cite{araujo2001successive, gillis2014fast} which we refer to as the generalized successive projection algorithm (GSPA) and that requires $O(mnr)$ operations as for most NMF algorithms.

   In Section~\ref{sec:numexp}, we perform extensive numerical experiments on synthetic, document and image data sets. 
   In most cases, we will observe that GS-NMF algorithms are able to compute decompositions with the same number of rank-one factors but with a lower approximation error than separable NMF algorithms.

\section{Properties of GS Matrices} \label{sec:prop}

Let us first show a simple property.
\begin{property}[Pattern of zeros] \label{prop:pattern}
    Let $M \in \mathbb{R}^{m\times n}_+$
be $(r_1,r_2)$-separable as described in Definition~\ref{def:gsmat1} so that $M=M(:,\mathcal{K}_1)P_1+P_2M(\mathcal{K}_2,:)$.
Then $M(\mathcal{K}_2, \mathcal{K}_1)
= 0_{r_2,r_1}$.
\end{property}
\begin{proof}
According to~\eqref{Gsep}, we have
\begin{align*}
M(\mathcal{K}_2, \mathcal{K}_1)
& =
M(\mathcal{K}_2, \mathcal{K}_1) P_1(:,\mathcal{K}_1),
+ P_2 (\mathcal{K}_2,:) M(\mathcal{K}_2,\mathcal{K}_1) \\
& =
M(\mathcal{K}_2, \mathcal{K}_1) + M(\mathcal{K}_2,\mathcal{K}_1),
\end{align*}
since, by definition, $P_1(:,\mathcal{K}_1) = I_{r_1}$ and $P_2(\mathcal{K}_2,:) = I_{r_2}$.
This implies that $M(\mathcal{K}_2,\mathcal{K}_1) = 0$.
\end{proof}

Intuitively,
in terms of topic modeling for example,
Property~\ref{prop:pattern} means that a pure document about a topic cannot contain an anchor word from another topic. \\

Let us provide two equivalent characterization of GS matrices.

\begin{property}[Equivalent characterization 1] \label{prop:eqiv1}
A matrix $M\in \mathbb{R}^{m\times n}_+$
is $(r_1,r_2)$-separable if and only if it can be written as
\begin{equation} \label{M}
 M=\Pi_r\left(
              \begin{array}{cc}
                 W_1 & W_1H_1+W_2H_2 \\
                 0_{r_2,r_1} & H_2 \\
              \end{array}
            \right)\Pi_c,
 \end{equation}
 for some permutations matrices $\Pi_c \in \{0,1\}^{n\times n}$ and $\Pi_r \in \{0,1\}^{m\times m}$, and for some nonnegative matrices
  $W_1 \in \mathbb{R}^{(m-r_2) \times r_1}_+$,
 $H_1 \in \mathbb{R}^{r_1\times (n-r_1)}_+$,
  $W_2 \in \mathbb{R}^{(m-r_2)\times r_2}_+$ and
  $H_2 \in \mathbb{R}^{r_2 \times (n-r_1)}_+$.
 \end{property}
 \begin{proof}
 This follows directly from Property~\ref{prop:pattern} and Definition~\ref{def:gsmat1}.
 The permutation $\Pi_c$ is chosen such that it moves the columns of $M$ corresponding to $\mathcal{K}_1$ in the first $r_1$ positions,
and the permutation $\Pi_r$ is chosen such that it moves the rows of $M$ corresponding to $\mathcal{K}_2$ in the last $r_2$ positions.
After these permutations,
there is $r_2$-by-$r_1$ block of zeros at the bottom left of $\Pi_r^T M \Pi_c^T$ since $M(\mathcal{K}_2,\mathcal{K}_1)=0$ (Property~\ref{prop:pattern}).
Moreover, since $M
= M(:,\mathcal{K}_1)P_1 + P_2 M(\mathcal{K}_2,:)$ for some nonnegative matrices $P_1$ and $P_2$,
we can take $W_1, H_1, W_2$ and $H_2$ such that
$M(:,\mathcal{K}_1) = [W_1; 0_{r_2,r_1}]$,
$P_1 = [I_{r_1} \, H_1] \Pi_c$,
$M(\mathcal{K}_2,:) = [0_{r_2,r_1} \, H_2]$,
and
$P_2 = \Pi_r [W_2; I_{r_2}]$.
  \end{proof}

As explained in the introduction, a matrix $M$ is $r$-separable if and only if it can be written as $M = MX$ where $X$ has $r$ non-zero rows. A similar characterization is possible for GS matrices.

 \begin{property}[Equivalent characterization 2] \label{prop:eqiv2}
A matrix $M\in \mathbb{R}^{m\times n}_+$
is $(r_1,r_2)$-separable if and only if it can be written as
$$
M=MX+YM,
$$
where
\begin{align}
X& =\Pi^{T}_c\left(
             \begin{array}{cc}
               I_{r_1} &H_1 \\
               0_{n-r_1,r_1} &0_{n-r_1,n-r_1} \\
             \end{array}
           \right)\Pi_c, \nonumber \\
           Y& = \Pi_r\left(
             \begin{array}{cc}
               0_{m-r_2,m-r_2}& W_2 \\
               0_{r_2,m-r_2} &  I_{r_2} \\
             \end{array}
           \right)\Pi^{T}_r, \label{eq:XY}
\end{align}
for some permutations matrices $\Pi_c \in \{0,1\}^{n\times n}$ and $\Pi_r \in \{0,1\}^{m\times m}$, and for some $H_1 \in \mathbb{R}^{r_1\times (n-r_1)}_+$ and $W_2 \in \mathbb{R}^{(m-r_2)\times r_2}_+$.
\end{property}
\begin{proof}
By Property~\ref{prop:eqiv1}, the matrix $M$ is $(r_1,r_2)$-separable if and only if there exist some permutation matrices $\Pi_c \in \{0,1\}^{n\times n}$ and $\Pi_r \in \{0,1\}^{m\times m}$ such that
\begin{equation}
 M=\Pi_r\left(
              \begin{array}{cc}
                 W_1 & W_1H_1+W_2H_2 \\
                 0_{r_2,r_1} & H_2 \\
              \end{array}
            \right)\Pi_c,
 \end{equation}
 for some $W_1 \in \mathbb{R}^{(m-r_2)\times r_1}_+$, $H_1 \in \mathbb{R}^{r_1\times (n-r_1)}_+$, $W_2 \in \mathbb{R}^{(m-r_2)\times r_2}_+$, $H_2 \in \mathbb{R}^{r_2\times (n-r_1)}_+$.
 Letting $\tilde{M}=\Pi^T_rM\Pi^T_c$ hence $M=\Pi_r\tilde{M}\Pi_c$, we have that $\tilde{M}$ is equal to
 $$
  \tilde{M}\left(
             \begin{array}{cc}
               I_{r_1} &H_1 \\
               0_{n-r_1,r_1} &0_{n-r_1,n-r_1} \\
             \end{array}
           \right)+\left(
             \begin{array}{cc}
               0_{m-r_2,m-r_2}& W_2 \\
               0_{r_2,m-r_2} &  I_{r_2} \\
             \end{array}
           \right) \tilde{M}.
$$
\end{proof}

In practice, given a GS matrix,
it is important to decompose it as a $(r_1,r_2)$-separable matrix with minimal value for $r_1+r_2$ since this compresses the data the most.
In the following, minimal $(r_1,r_2)$-separable matrices are defined.
 \begin{defn}
 A matrix $M$ is a minimal $(r_1,r_2)$-separable if $M$ is $(r_1,r_2)$-separable and $M$ is not $(r'_1,r'_2)$-separable for any $r_1'+r_2'< r_1+r_2$.
 \end{defn}

 By property~\ref{prop:eqiv2}, finding minimal GS decompositions is equivalent to finding  $X$ and $Y$ that satisfy~\eqref{eq:XY} and such that the number of non-zero rows of $X$ and non-zero columns of $Y$ is minimized.

\begin{property}[Idealized model] Let $M$ be minimal $(r_1,r_2)$-separable, and let $(X^*,Y^*)$ be an optimal solution of
\begin{align}
  \min_{X\in \mathbb{R}^{n\times n}_+,Y\in \mathbb{R}^{m\times m}_+} & \|X\|_{row,0}+\|Y\|_{col,0} \nonumber \\
 &  \text{ such that } \; M=MX+YM, \label{eq:combinmodel}
 \end{align}
where $ \|X\|_{row,0}$ is equal to the number of nonzero rows of $X$ and $\|Y\|_{col,0}$ is equal to the number of nonzero columns of $Y$.
Let also $\mathcal{K}_1$ correspond to the indices of the non-zero rows of $X^*$ and
$\mathcal{K}_2$ to the indices of the non-zero columns of $Y^*$. If $\rank(M) = r_1+r_2$, we have
\[
\|X^*\|_{row,0}+\|Y^*\|_{col,0}
= |\mathcal{K}_1| + |\mathcal{K}_2|
= r_1+r_2.
\]
\end{property}
\begin{proof}
By Property~\ref{prop:eqiv2}, an $(r_1,r_2)$-separable matrix can be written as $M=MX+YM$ where the number of non-zero rows of $X$ and non-zero columns of $Y$ is $r_1+r_2$.
Hence, by optimality of
$(X^*,Y^*)$, we have
$|\mathcal{K}_1| + |\mathcal{K}_2|
\leq r_1+r_2$.

Moreover, since $\rank(M) = r_1+r_2$, we must have $|\mathcal{K}_1| + |\mathcal{K}_2|
\geq r_1+r_2$.
\end{proof}

Some remarks are in order:
\begin{itemize}

\item As opposed to separable NMF, due to the non-uniqueness of GS-NMF, $|\mathcal{K}_1|$ is not necessarily equal to $r_1$ and $|\mathcal{K}_2|$ to $r_2$; see Property~\ref{prop:constrnonuniq} below.

\item Unfortunately, solving~\eqref{eq:combinmodel} does not guarantee $X$ and $Y$ to have the form~\eqref{eq:XY} where $X$ and $Y$ contain the identify matrix as a submatrix.
This is why we need the condition $\rank(M) = r_1+r_2$.

\item Of course, \eqref{eq:combinmodel} is a difficult combinatorial problem. We will consider in Section~\ref{sec:optim} a convex relaxation.
Before doing that, we first present several other interesting properties of GS matrices.

\end{itemize}


From a practical point of view, GS matrices will be particularly interesting when they allow to compress the data significantly more than separable matrices. In other words, it would be interesting to know whether there exists $(r_1,r_2)$-separable matrices which are not $(r,0)$- nor $(0,r)$-separable for $r \gg r_1+r_2$.
In fact, this is the case; see Property~\ref{prop:comp} for the case $r_1=r_2=3$ and $r = \min(m-1,n-1)$.
First let us show the following lemma.

\begin{lemma} \label{lem:21sep}
There exist 3-by-$n$ matrices
that are $(2,1)$-separable but not $(n-1,0)$-separable.
\end{lemma}
\begin{proof}
 Consider the 3-by-$n$ matrix
 \begin{equation} \label{Mn21sep}
   M_{n} = \left(
       \begin{array}{ccccc}
         1 & 0 & \frac{1}{2} & 0 & x^T \\
        0 & 1 &0 & \frac{1}{2} & y^T \\
         0 & 0 & \frac{1}{2} & \frac{1}{2} & z^T \\
       \end{array}
     \right)
 \end{equation}
 where $x, y, z \in \mathbb{R}^{n-4}$
are such that $(x_i,y_i,z_i)$ for $i=1,2,\dots,n-4$ are defined as $0<x_i<\frac{1}{2}$ and $x_i\neq x_j$ for all $i\neq j$,
$y_i=2(\frac{1}{2}-x_i)^2$, $z_i=1-x_i-y_i$.
The points $(x_i,y_i,z_i)$ are distinct points on a curve on the unit simplex hence such points cannot be written as conic combinations of any other points on that curve. In fact, since the entries of the vectors $(x_i,y_i,z_i)$
 sum to one for $i=1,2,\dots,n-4$,
 the weights in such a conic combination would also have to sum to one hence such a conic combination would actually be a convex combination. Clearly, distinct points on the circle $(x_i,y_i)$'s are not convex combination of one another; in other words, every such point is a vertex of their convex hull.

Note also that the third and fourth column of $M_{n}$ are the two extreme points of that curve. The first column of $M_{n}$ also cannot be written as a convex combination of all the other columns since $z_i \neq 0$ for all $i$. This implies that $M_{n}$ is not $(n-1, 0)$ separable: every column of $M_{n}$ is an extreme ray of the cone spanned by the columns of $M_{n}$.
Moreover, $M_{n}$ is $(2,1)$-separable since $M_{n}(1:2,1:2) = I_r$ while the third row can be approximated by itself: we have
\[
M_{n}=M_{n}(:,1:2)P_1 + P_2 M_{n}(3,:),
\]
where $P_1=M_{n}(1:2,:)$, $P_2=(0,0,1)^T$.
\end{proof}


\begin{property}[Compression] \label{prop:comp}
There exist $m$-by-$n$ matrices that are $(3,3)$-separable but not $(n-1,0)$- nor $(0,m-1)$-separable.
\end{property}
\begin{proof}
Let $M_{n}$ be a 3-by-$n$ matrix and $M_{m}$ be a 3-by-$m$ matrix constructed as in~\eqref{Mn21sep}. Let us also construct the $(m+3)$-by-$(n+3)$ matrix
 \begin{equation} \label{eq:mat33}
   M = \left(
       \begin{array}{cc}
        0_{3,3}          & M_{n} \\
        M_{m}^T & 0_{m,n}
       \end{array}
     \right).
 \end{equation}
By Lemma~\ref{lem:21sep},
$M$ is (3,3)-separable (note that the corresponding GS decomposition is not unique since $M_{n}$ is (2,1)- and (0,3)-separable),
while not being ($n+2$,0)-separable nor (0,$m+2$)-separable. In fact, assume $M$ is (0,$m+2$)-separable.
Observe that any row that would be selected from the first $3$ rows (resp.\@ last $m$ rows) of $M$ cannot be used to reconstruct any of the last $m$ rows (resp.\@ first $3$ rows) of $M$ using a positive weight because of the zeros in the last positions (resp.\@ in the first positions). Hence a (0,$m+2$)-separable decomposition of $M$ would imply that either $M_{m}^T$ is (0,$m-1$)-separable, a contradiction with Lemma~\ref{lem:21sep}, or that $M_{n}$ is (0,2)-separable which is not possible since $\rank(M_{n}) = 3$.
The same observation holds for the columns, by symmetry of the problem.
\end{proof}

The next property is rather straightforward but we state here for completeness. It shows that generalized separability is invariant to scaling.

\begin{property}[Scaling] \label{prop:scal}
The matrix $M$ is $(r_1,r_2)$-separable if and only if $D_1MD_2$
is $(r_1,r_2)$-separable for any diagonal matrices $D_1$ and $D_2$ whose diagonal elements are positive.
\end{property}
\begin{proof}
Let $M$ be $(r_1,r_2)$-separable with $M = M(:,\mathcal{K}_1) P_1
+ P_2 M(\mathcal{K}_2)$ with $|\mathcal{K}_1| = r_1$ and $|\mathcal{K}_2| = r_2$.
Multiplying on both sides by $D_1$ and $D_2$, we obtain
\[
D_1M D_2
= D_1M(:,\mathcal{K}_1) P_1 D_2
+ D_1 P_2 M(\mathcal{K}_2,:) D_2.
\]
Denoting $\tilde{M} = D_1 M D_2$,
$\tilde{P}_1 =
D_2(\mathcal{K}_1,\mathcal{K}_1)^{-1} P_1 D_2$ and
$\tilde{P}_2 =
D_1 P_2 D_1(\mathcal{K}_2,\mathcal{K}_2)^{-1}$, we have
\[
\tilde{M} = \tilde{M}(:,\mathcal{K}_1) \tilde{P}_1
+ \tilde{P}_2 \tilde{M}(\mathcal{K}_2,:),
\]
where
$\tilde{P}_1(\mathcal{K}_1,\mathcal{K}_1) = I_{r_1}$ and
$\tilde{P}_2(\mathcal{K}_2,\mathcal{K}_2) = I_{r_2}$; hence $\tilde{M}$ is $(r_1,r_2)$-separable.
The proof in the other direction is the same since $\tilde{M} = D_1 M D_2$ is the diagonal scaling of $M$ using the inverses of $D_1$ and $D_2$.
\end{proof}


\subsection{Unicity of GS decompositions} \label{sec:uniq}

As opposed to separable NMF, GS-NMF does not necessarily admit a unique solution (up to scalings and permutations of the rank-one factors).
In other words, for a minimal $(r_1,r_2)$-separable matrix $M$, the way of picking the rows and columns of $M$ is not necessarily unique:  it may also be $(r_3,r_4)$-separable with  $r_3+r_4=r_1+r_2$ where $r_1\neq r_3$ and $r_2\neq r_4$, or it can be $(r_1,r_2)$-separable with different selection of rows and columns; this is the case for example for the matrix $M$ in~\eqref{eq:mat33}.

The simplest cases are for rank-one and rank-two matrices.
 \begin{property}[Rank-one matrices] \label{prop:nonuniqrank1}
 Any nonnegative rank-one matrix is (1,0)- and (0,1)-separable.
 \end{property}
 \begin{proof}
 This follows directly from the fact that all rows (resp.\@ columns) of a rank-one matrix are multiple of one another.
 \end{proof}

 \begin{property}[Rank-two matrices] \label{prop:nonuniqrank2}
 Any nonnegative rank-two matrix is (2,0)- and (0,2)-separable.
 \end{property}
 \begin{proof}
 This follows from the fact that any nonnegative rank-two matrix is 2-separable~\cite{Tho74}.
 The reason is that a two-dimensional cone is always spanned by its two extreme rays.
 \end{proof}

Examples can be constructed for any values of $(r_1,r_2)$.

 \begin{property}[Construction of non-unique minimal $(r_1,r_2)$-separable matrices] \label{prop:constrnonuniq}
For any $(r_1,r_2)$, we can construct  minimal
$(r_1,r_2)$-separable matrices such that
they are also minimal $(r_3,r_4)$-separable with $r_1+r_2 = r_3+r_4$, $r_3 \neq r_1$ and $r_4 \neq r_2$.
  \end{property}
\begin{proof}
Let  $r_1>r_3$, $r_2<r_4$ and $r_1+r_2 = r_3+r_4$.
Let also $M_{11}\in \mathbb{R}^{(m-r_4)\times  r_3}_+$,
$M_{22}\in \mathbb{R}^{(r_4-r_2)\times  (r_1-r_3)}_+$ and  $M_{33}\in \mathbb{R}^{r_2\times (n-r_1)}_+$ be any nonnegative matrices.
Let us construct $M$ as follows:
 $$
M=\left(
             \begin{array}{ccc}
               M_{11} & 0_{m-r_4, r_1-r_3} &M_{13} \\
               0_{r_4-r_2,r_3} & M_{22} & M_{23} \\
               0_{r_2,r_3}& 0_{r_2,r_1-r_3} & M_{33}\\
             \end{array}
           \right),
$$
 where
\[
M_{13}=M_{11}X_1+Y_1M_{33} \in \mathbb{R}^{(m-r_4)\times (n-r_1)}_+, \text{ and }
\]
\[
M_{23}=M_{22}X_2+Y_2M_{33}  \in \mathbb{R}^{(r_4-r_2)\times (n-r_1)}_+,
\]
for any $X_1\in \mathbb{R}^{r_3\times (n-r_1)}_+$, $Y_1\in \mathbb{R}^{(m-r_4)\times r_2}_+$, $X_2\in \mathbb{R}^{(r_1-r_3)\times {(n-r_1)} }_+$, $Y_2\in \mathbb{R}^{(r_4-r_2)\times r_2}_+$.

We have that $M$ is $(r_1,r_2)$-separable
where $\mathcal{K}_1$ contains the first $r_1$ columns of $M$ and $\mathcal{K}_2$ contains the last $r_2$ rows of $M$ since
\begin{eqnarray*}
  \left(
  \begin{array}{c}
    M_{13} \\
    M_{23} \\
  \end{array}
\right)&=&\left(
            \begin{array}{cc}
              M_{11} &0_{m-r_4, r_1-r_3} \\
              0_{r_4-r_2,r_3} & M_{22} \\
            \end{array}
          \right)\left(
                   \begin{array}{c}
                     X_1 \\
                     X_2 \\
                   \end{array}
                 \right)\\
                 & & +\left(
                           \begin{array}{c}
                             Y_1 \\
                             Y_2 \\
                           \end{array}
                         \right)M_{33}
\end{eqnarray*}

Similarly, we have that $M$ is $(r_3,r_4)$-separable
where $\mathcal{K}_1$ contains the first $r_3$ columns of $M$ and $\mathcal{K}_2$ contains the last $r_4$ rows of $M$ since $(0_{m-r_4, r_1-r_3} \; M_{13})$ is equal to
\begin{equation*}
  M_{11} \left(
     \begin{array}{cc}
       0 & X_1 \\
     \end{array}
   \right)
    + \left(
     \begin{array}{cc}
       0 & Y_1 \\
     \end{array}
   \right)\left(
            \begin{array}{cc}
              M_{22} & M_{23} \\
              0_{r_2,r_1-r_3} & M_{33} \\
            \end{array}
          \right).
\end{equation*}
\end{proof}

The simplest example is for a 3-by-3 matrix
that is $(2,1)$- and $(1,2)$-separable, and also trivially $(3,0)$- and $(0,3)$-separable:
\[
\left(
\begin{array}{ccc}
1 & 0 & 2 \\ 0 & 1 & 2 \\ 0 & 0 & 1
\end{array}
\right).
\]
We simply took $M_{11} = M_{22} = M_{33} = X_1 = X_2 = Y_1 = Y_2 = 1$ in Property~\ref{prop:constrnonuniq}. \\

However, it is possible to guarantee uniqueness of GS decompositions.
A possible way is to have a single pattern of zeros which is large enough.

\begin{property}[Condition for uniqueness] \label{prop:conduniq}
Let $M$ be minimal $(r_1,r_2)$-separable
and let $M$ not be $(r_1+r_2,0)$-separable nor $(0,r_1+r_2)$-separable.
If $M$ does not contain a pattern of zeros of size $r_3 r_4$ with $r_1+r_2 = r_3+r_4$, except for $M(\mathcal{K}_2,\mathcal{K}_1) = 0$, then  $M$ admits a unique GS decomposition of size $(r_1,r_2)$.
\end{property}
\begin{proof}
  This follows directly from Property~\ref{prop:pattern}.
 \end{proof}

Identifiability is a key aspect of NMF models; see the recent survey~\cite{xiao2019uniq} on this topic and the references therein.
Let us discuss this aspect and how GS-NMF allows to resolve this issue (given that the input matrix is in fact a GS matrix).
Given a nonnegative matrix $X = WH$ for $W \geq 0$ and $H \geq 0$, the conditions for such a factorization to be unique, up to permutation and scaling of the rank-one factors $W(:,k)H(k,:)$ ($1 \leq k \leq r$), are rather strong and in general not met in practice. In a few words, these conditions require $W$ \textit{and} $H$ to be sufficiently sparse.
Therefore, in practice, it is crucial to consider reglularized NMF models; for example adding sparsity constraints on the factors $W$ and/or $H$~\cite{hoyer2004non}.
A key NMF model that leads to identifiability of NMF under very mild conditions is minimum-volume NMF (MV-NMF)\footnote{In fact, the sufficiently scattered condition which is sufficient to guarantee MV-NMF to be identifiable is conjectured to also be necessary~\cite{xiao2019uniq}.} which requires the convex hull of the columns of $W$ to have the smallest possible volume while the columns of $H$ are normalized to have unit $\ell_1$ norm~\cite{miao2007endmember}. Under this model, NMF is identifiable given that the matrix $H$ is sufficiently sparse; this condition is referred to as the sufficiently scattered condition~\cite{huang2013non, lin2015identifiability, fu2018identifiability, xiao2019uniq}.
We will compare MV-NMF with GS-NMF in the numerical experiments (Section~\ref{sec:numexp}) and observe that GS-NMF is able to recover the true factors $W$ and $H$ that generated the GS matrix while standard NMF and MV-NMF fail to do so.
This means that GS-NMF really brings a new class of identifiable NMF solutions. The reason is that GS matrices satisfy different conditions than the sufficiently scattered condition. This allows proper algorithms, like the ones we will propose in the next sections, to take advantage of these properties hence making them able to recover the true $W$ and $H$.
We believe that this is a key asset of GS-NMF.
For example, in audio source separation,
it is reasonable to assume that the input matrix is a GS matrix (see the Introduction), while it might not be separable nor satisfy the sufficiently scattered condition.


 \section{Convex Optimization Model and Fast Gradient Method} \label{sec:optim}

In real data sets, due to the presence of noise (and model misfit),
the model \eqref{eq:combinmodel} should be modified to
 \begin{align}
  \min_{X\in \mathbb{R}^{n\times n}_+,Y\in \mathbb{R}^{m\times m}_+} & \|X\|_{row,0}+\|Y\|_{col,0} \nonumber \\
  & \text{ such that }
   \|M-MX-YM\| \leq \epsilon, \label{ee1}
 \end{align}
where $\epsilon$ denotes the noise level.
The norm of the residual $\|M-MX-YM\|$ can be chosen according to the noise statistic.
In this paper, we will consider the Frobenius norm, that is,
$\|M-MX-YM\|_F = \sum_{i,j} (M-MX-YM)_{i,j}^2$; see for example~\cite{dikmen2015learning} for a discussion on the choice of the objective function.

\subsection{Convex optimization model}

As it is challenging to solve \eqref{ee1}, it can be relaxed to a convex optimization model as follows:
 \begin{align}
  \min_{X\in \mathbb{R}^{n\times n}_+,Y\in \mathbb{R}^{m\times m}_+}& \|X\|_{1,q}+\|Y^T\|_{1,q}\nonumber \\
  \text{ such that } &
   \|M-MX-YM\|\leq \epsilon,\label{e2}
 \end{align}
where $\|X\|_{1,q}:=\sum^n_{i=1}\|X(i,:)\|_q$ and $\|Y^T\|_{1,q}:=\sum^n_{i=1}\|Y(:,i)\|_q$.
The quantities $\|X\|_{1,q}$ and $\|Y^T\|_{1,q}$ are the $\ell_1$ norm of the vector containing the $l_q$ norms of the rows of $X$ and the columns of $Y$, respectively.
The model aims to generate a matrix $X$ with only a few non-zero rows and a matrix $Y$ with only a few non-zero columns.
This model is a generalization of separable NMF convex relaxations: $q=2$ was proposed in~\cite{elhamifar2012see}, while $q=+\infty$ was proposed in~\cite{esser2012convex}.
In fact, \eqref{e2} coincides with the models
from~\cite{esser2012convex, elhamifar2012see} by taking $Y=0$.

The rationale behind this model is that the $\ell_1$ norm is the largest convex function smaller than $\ell_0$ norm on the $\ell_\infty$ ball; see for example~\cite{recht2010guaranteed}.
In other terms, $\|X\|_{1,q} \leq \|X\|_{row,0}$ as long as
$\|X(i,:)\|_q \leq 1$ for all $i$.

Considering $q = +\infty$, $\|X(i,:)\|_q \leq 1$ holds for example for $X \leq 1$. This can be assumed without loss of generality given that the input matrix is properly scaled.
\begin{defn} \label{def:scaled}
The matrix $M$ is scaled if $||M(:,j)||_1 = k_1$ for all $j$ and
$||M(i,:)||_1 = k_2$  for all $i$, for some $n k_1 = m k_2 > 0$.
\end{defn}
Given a nonnegative matrix $M$, it is in most cases possible to scale it, that is, find diagonal matrices $D_r$ and $D_c$ such that $M_s = D_r M D_c$ is  scaled. It requires that the matrix $M$ has sufficiently many non-zero elements.
When the matrix is scalable, the algorithm that alternatively scales the columns and rows of $M$ will converge to a scaled matrix. We refer the reader to~\cite{knight2008sinkhorn, olshen2010successive} for more details on this topic.

We have the following property.
\begin{property} \label{prop:smaller1}
Let $M$ be a scaled $(r_1,r_2)$-separable matrix.
Then $M$ can be decomposed as in~\eqref{eq:XY} with
\[
X(i,j) \leq X(i,i) \leq 1
  \text{ for } 1 \leq i,j \leq n,
   \]
\[
  \text{ and } \quad
  Y(l,t) \leq Y(t,t) \leq 1  \text{ for }
  1 \leq l,t \leq m.
\]
\end{property}
\begin{proof}
Using Property~\ref{prop:eqiv1} we have, after proper permutations of the columns and rows of $M$, that
$$M=\left(
     \begin{array}{cc}
       W_1 & W_1H_1+W_2H_2 \\
       0_{r_2,r_1} & H_2 \\
     \end{array}
   \right).
$$
Since $M$ is scaled, we have
$e^TM(:,j)=e^T W_1(:,j)=k_1$ for $1 \leq j \leq r_1$, where $e$ is the vector of all one of appropriate dimension.
For $j=r_1+1,\dots, n$, we have
\[
k_1
= e^T M(:,j)
\geq
e^T W_1 H_1(:,j-r_1)
  = k_1 e^T H_1(:,j-r_1),
 \]
 since all matrices involved are nonnegative.
 This implies that $H_1 \leq 1$.
 In fact, this implies the stronger condition $||H_1(:,j)||_1 \leq 1$ for all $j$.
By symmetry, the same result holds for $W_2$, that is, $W_2 \leq 1$ and $||W_2(i,:)||_1 \leq 1$ for all~$i$.
Therefore, up to permutations, using the same derivations as in Property~\ref{prop:eqiv2}, we have
$$
M=M\left(
     \begin{array}{cc}
       I & H_1 \\
       0 & 0 \\
     \end{array}
   \right)+\left(
     \begin{array}{cc}
       0 & I \\
       0 & W_2 \\
     \end{array}
   \right)M,
$$
where $H_1 \leq 1$ and $W_2 \leq 1$.
\end{proof}


In this paper, we focus on another convex model to tackle GS-NMF. For a scaled GS matrix, it can be written as follows:
\begin{equation}\label{e3}
\begin{split}
  \min_{X\in \mathbb{R}^{n\times n}_+,Y\in \mathbb{R}^{m\times m}_+} & \tr(X)+ \tr(Y),\\
  \mbox{such that} & \quad  \|M-MX-YM\|\leq \epsilon, \\
  &  \quad X(i,j) \leq X(i,i) \leq 1
  \text{ for } 1 \leq i,j \leq n,  \\
 & \quad  Y(l,t) \leq Y(t,t) \leq 1
 \text{ for } 1 \leq l,t \leq m.
\end{split}
\end{equation}
This model is the generalization of the model from~\cite{gillis2013robustness} for separable matrices, which is an improvement of the model from~\cite{recht2012factoring}.
The rationale behind this model is the following.
Since $X$ is nonnegative, minimizing its trace is equivalent to minimize the $\ell_1$ norm of its diagonal entries, that is, $\tr(X) = ||\diag(X)||_1$. Hence~\eqref{e3} promotes solutions whose diagonal is sparse.
Then, the constraints $X(i,j) \leq X(i,i)$ for all $i,j$ impose that the largest entry in each row is the corresponding diagonal entry. Hence, if a diagonal entry is equal to zero, the entire row is zero. This makes this model generate solutions that tend to be row sparse.
Note that for any feasible solution $X$ of~\eqref{e3}, we have
$\tr(X) \leq \|X\|_{row,0}$.
In fact, since $0 \leq X(i,j) \leq X(i,i)$ for all $i,j$,
$\|X\|_{1,\infty} = \tr(X)$ for any $X$. Moreover, since $X \leq 1$, $\|X\|_{1,\infty} \leq \|X\|_{row,0}$.
By symmetry, we also have $\tr(Y) \leq \|Y\|_{col,0}$.

Unfortunately, as opposed to the model for separable matrices, we were not able to show that~\eqref{e3} is provably able to recover the correct set of column and row indices, even in the presence of low-noise levels.
This is an important direction of future research. However, in the numerical experiments in Section~\ref{sec:numexp},
this model performs this task perfectly in all tested scenarios (see Figures~\ref{fullacc} and~\ref{acc-mid}).

The model~\eqref{e3} can easily be generalized for non-scaled $M$; see Section~\ref{sec:fgm}.
Compared to~\eqref{e2}, it has an important advantage: It is a smooth optimization problem, and the projection onto the feasible set can be performed efficiently, in $\mathcal{O}(n^2 \log n+ m^2 \log m)$ operations~\cite{gillis2018fast}.
Therefore, we can easily design first-order optimization method with strong convergence guarantees; see Section~\ref{sec:fgm}.

\subsection{Fast Gradient Method for GS-NMF} \label{sec:fgm}


 Let us generalize the model~\eqref{e3} to non-scaled matrices, as done for separable matrices in~\cite{gillis2014robust}.
 Using essentially a similar argument as in the proof of Property~\ref{prop:smaller1}, we have for a GS matrix $M$ that for all $j$
 \[
 M(:,j)
 \leq M(:,\mathcal{K}_1) P_1(:,j)
 \leq M(:,\mathcal{K}_1(k)) P_1(k,j) \text{ for all $k$},
 \]
 since $M = M(:,\mathcal{K}_1) P_1 + P_2 M(\mathcal{K}_2,:)$.
 Taking the $\ell_1$ norm on both sides, we have
 \[
 P_1(k,j) \leq \frac{|| M(:,j) ||_1}{||M(:,\mathcal{K}_1(k))||_1} \text{ for all } k,j.
 \]
 A similar observation can be made for $P_2$, which leads to the generalization of~\eqref{e3} for non-scaled matrices:
 \begin{align}
 \min_{X\in \Omega_1, Y\in \Omega_2} &\tr(X)+\tr(Y)\nonumber \\
  & \text{ such that }
   \|M-MX-YM\|\leq \epsilon,\label{e4}
 \end{align}
where the sets $\Omega_1$ and $\Omega_2$ are defined as
 $$
 \Omega_1=\{X\in \mathbb{R}^{n\times n}_+| X \leq 1, w_i X(i,j)\leq w_j X(i,i) \text{ for all } i,j\},$$
 $$
 \Omega_2=\{ Y \in \mathbb{R}^{m\times m}_+
 | Y \leq 1, \hat{w}_t Y(l,t) \leq \hat{w}_l Y(t,t)  \text{ for all } l,t\},
 $$
  where the vector ${w}\in \mathbb{R}^{n}_+$ contains the $l_1$ norm of the columns $M$, that is, $w_j=\|M(:,j)\|_1$ for all $j=1,\dots,n$,
  and  the vector $\hat{w}\in \mathbb{R}^{m}_+$ contains the $l_1$ norm of the rows of $M$, that is,
  $\hat{w}_l=\|M(l,:)\|_1$ for all $l=1,\cdots,m$.

To solve the smooth convex problem~\eqref{e4},
interior-point methods can be used for example using SDPT3~\cite{toh1999sdpt3}.
However using such second-order method to solve~\eqref{e4} which has $n^2 + m^2$ variables and as many constraints would be numerically expensive.
Moreover, in our case, high accuracy solutions are not crucial: the main goal of solving~\eqref{e4} is to identify the important columns and rows of $M$ which correspond to the largest entries in the diagonal entries of $X$ and $Y$.
Therefore, we use Nesterov's optimal first-order method~\cite{nes83, nes04}, namely, a fast gradient method, similarly as done
in~\cite{gillis2018fast} for separable matrices. Here ``fast'' refers to the fact that it attains
the best possible convergence rate of $\mathcal{O}(1/k^2)$ in the
first-order regime.
To do so, we consider the penalized version of~\eqref{e4}:
 \begin{align}
 &\min_{X\in \Omega_1, Y\in \Omega_2}
F(X,Y),
\label{mod1}
 \end{align}
 with $F(X,Y) = \frac{1}{2}\|M-MX-YM\|^2_F + \lambda \big(\tr(X)+\tr(Y) \big)$,
where $\lambda > 0$ is a penalty parameter which balances the importance between the approximation error $\|M-MX-YM\|^2_F$ and the sum of the traces of $X$ and $Y$.

To initialize $X$ and $Y$
and set the value of $\lambda$, we adopt the following strategy described in Algorithm~\ref{algo:initfgm}, similarly as in~\cite{gillis2018fast}:
\begin{itemize}

  \item Extract a subset $\mathcal{K}_1$ of columns and a subset $\mathcal{K}_2$ of rows of $M$
  such that $|\mathcal{K}_1|+\mathcal{K}_2=r$
  using the heuristic algorithm referred to as GSPA; see Section~\ref{sec:heur}.

  \item Compute the corresponding optimal weights $(P_1^*,P_2^*)$ which is the solution to
 \begin{equation} \label{eq:p1p2}
  \mathop{\min}_{
 \begin{array}{c}
 P_1 \in \mathbb{R}^{r_1 \times n}_+, \\
  P_2 \in \mathbb{R}^{m \times r_2}_+
  \end{array}
  }
  \|M-M(:,\mathcal{K}_1)P_1 - P_2 M(\mathcal{K}_2,:)\|^2_F.
 \end{equation}

  We used the coordinate descent implemented in~\cite{gillis2012accelerated}.

  \item  Define $X_0(\mathcal{K}_1,:)=P_1^*$ and $Y_0(:,\mathcal{K}_2)=P_2^*$,
  while $X_0(i,:)=0$ for $i \notin \mathcal{K}_1$
  and $Y_0(:,j)=0$ for  $j \notin \mathcal{K}_2$.

  \item  Set $\lambda = \tilde{\lambda} \frac{\|M-MX_0-Y_0M\|}{2r}$, where $r=r_1+r_2$ and some $\tilde{\lambda}$. Typically, $\tilde{\lambda} \in [10^{-3},10]$ works well.

\end{itemize}

\begin{algorithm}[h]
\caption{Initialization for GS-FGM \label{algo:initfgm}}
\begin{algorithmic}[1]
\REQUIRE $M\in\mathbb{R}_{+}^{m\times n}$, $r$,
$\tilde{\lambda} > 0$.

\ENSURE Initial solution $(X_0,Y_0) \in \mathbb{R}^{n \times n} \times \mathbb{R}^{m \times m}$ for~\eqref{mod1},
and parameter $\lambda$ balancing the two terms in the objective.

\STATE $(\mathcal{K}_1, \mathcal{K}_1)$ = GSPA($M,r$); see Algorithm~\ref{algo:gspa};

\STATE Compute $(P_1^*,P_2^*)$
  as a solution to~\eqref{eq:p1p2};

  \STATE $X_0(\mathcal{K}_1,:) = 0_{n,n}$;  $Y_0(:,\mathcal{K}_2) = 0_{m,m}$;

\STATE $X_0(\mathcal{K}_1,:) = P_1^*$;  $Y_0(:,\mathcal{K}_2) = P_2^*$;

 \STATE $\lambda = \tilde{\lambda} \frac{\|M-MX_0-Y_0M\|}{2r}$.

\end{algorithmic}
\end{algorithm}

To solve model (\ref{mod1}), we employ Algorithm~\ref{algo:fgm} which is an optimal first-order method to minimize $F(X,Y)$ over the sets $\Omega_1$ and $\Omega_2$. To compute the Euclidean projection of $X$ on the set $\Omega_1$ and of $Y$ on the set $\Omega_2$, we use the method proposed in~\cite{gillis2018fast}, which only requires $\mathcal{O}(n^2\log n)$ and $\mathcal{O}(m^2\log m)$ operations for the projection of $X$ and $Y$, respectively.

The main computational cost of Algorithm~\ref{algo:fgm} resides in lines 2, 7 and 9.
For line 2, the maximum singular value of $M$ can be well approximated by the power method which needs $\mathcal{O}(mn)$ operations.
In line 5,
the computation of the different matrix products require $\mathcal{O}(mn^2 + m^2n)$ operations.
For line 7, the projections of $X$ and $Y$ require
$\mathcal{O}(n^2 \log n + m^2 \log m)$~\cite{gillis2018fast}.
Finally, Algorithm~\ref{algo:fgm}
requires $\mathcal{O}(m n^2 + m^2 n)$ operations, assuming $m \geq \log n$ and $n \geq \log m$.


\begin{algorithm}[h]
\caption{GS-NMF with a Fast Gradient Method (GS-FGM) \label{algo:fgm}}
\begin{algorithmic}[1]
\REQUIRE $M\in\mathbb{R}_{+}^{m\times n}$,
number $r_1$ of columns and $r_2$ of rows to extract,
and
maximum number of iterations $maxiter$.
\ENSURE Matrices $X$ and $Y$ solving~\eqref{mod1},
and a set $\mathcal{K}_1$ of column indices and a set $\mathcal{K}_2$ of row indices such that $\min_{P_1,P_2 \geq 0} ||M-M(:,\mathcal{K}_1)P_1+P_2M(\mathcal{K}_2,:)||_F$ is small.

\STATE \emph{\% Initialization }

\STATE $\alpha_0 \leftarrow 0.05$;
$L\leftarrow 2\sigma_{\max}(M)^2$; \\
Initialize $X$, $Y$ and $\lambda$;
see Algorithm~\ref{algo:initfgm}.


\FOR{$k$ = 1 : maxiter}

\STATE \emph{\% Keep previous iterates in memory}

\STATE $X_p\leftarrow X$; $Y_p\leftarrow Y$;

\STATE \emph{\% Gradient computation}

\STATE $\nabla_X F(X,Y)\leftarrow M^TMX+M^TYM-M^TM+\lambda I_n$; \\ $\nabla_Y F(X,Y)\leftarrow MXM^T+YMM^T-MM^T+\lambda I_m$;

\STATE \emph{\% Gradient step and projection}

\STATE $X_n \leftarrow \mathrm{P}_\Omega(X-\frac{1}{L}\nabla_X F(X,Y))$; \\
$Y_n \leftarrow \mathrm{P}_\Omega(Y-\frac{1}{L}\nabla_Y F(X,Y))$;

\STATE \emph{\% Acceleration / Momentum step}

\STATE $X\leftarrow X_n +\beta_k(X_n-X_p)$; \\
$Y\leftarrow Y_n +\beta_k(Y_n-Y_p)$; \\ where $\beta_k=\frac{\alpha_{k-1}(1-\alpha_{k-1})}{\alpha^2_{k-1}+\alpha_k}$ such that $\alpha_k\geq 0$ and $\alpha^2_k=(1-\alpha_k)\alpha^2_{k-1}$

\ENDFOR

\STATE $\mathcal{K}_1\leftarrow \text{post-process}(X,r_1)$; \\
$\mathcal{K}_2\leftarrow \text{post-process}(Y,r_2)$.
\end{algorithmic}
\end{algorithm}

We will refer to Algorithm~\ref{algo:fgm} as GS-FGM.
Note that the numbers $r_1$ and $r_2$ are given as input of Algorithm~\ref{algo:fgm}.
However, they can also be detected automatically by identifying the entries on the diagonals of $X$ and $Y$ above a certain threshold.
For simplicity, we will use the same two post-processing procedures as in~\cite{gillis2018fast}:
\begin{itemize}
 \item   For synthetic data sets, we simply pick the $r_1$ largest entries of the diagonals of $X$ and the $r_1$ largest entries of the diagonals of $Y$.

  \item For real data sets, it is also important to consider off-diagonal entries of $X$ and $Y$. The reason is that the input matrix can be far from being a GS matrix. For example, an outlying column will in general lead to a large diagonal entry in $X$ (since an outlier is in general not well approximated with other data points) while the other entries on the same row will be close to zero (since an outlier is in general useless to reconstruct other data points). This means that if a row of $X$ has many large entries, it is likely to be more important than a row with only a large diagonal entry. For this reason, we sort the columns of $M$ by applying SPA on $X^T$ as done in~\cite{gillis2018fast}; an similarly for $Y$ to sort the rows of $M$. It remains to decide how many column and row indices to pick in each of these ordered sets. To do so, we sequentially select a column or a row of $M$ as follows: at each step, we will select the column/row of $M$ such that the residual after projection onto its orthogonal complement is the smallest, and at the next step, we replace $M$ by the corresponding residual; this shares some similarity with the algorithm presented in the next section.
\end{itemize}

\section{Heuristic Algorithm for GS-NMF} \label{sec:heur}

Algorithm~\ref{algo:fgm} is computationally expensive, and does not scale linearly with the dimension of the input matrix. For large-scale problems, it would not be applicable.
When running on a standard computer, $m$ and $n$ should be limited to values below a thousand.
A possible way to overcome this issue is to preselect, a priori, a subset of columns and rows of $M$, reducing the number of variables; see Section~\ref{sec:doc} for a discussion.

In this section, we derive a fast heuristic algorithm for GS-NMF. It is inspired from one of the most widely used separable NMF algorithm, namely the successive projection algorithm (SPA).  SPA is essentially equivalent to QR with column pivoting; it was introduced in~\cite{araujo2001successive} in the contex of spectral unmixing but has been rediscovered many times; see the discussions in~\cite{ma2014signal, gillis2014and}.
Moreover, SPA is robust in the presence of noise~\cite{gillis2014fast}.
SPA assumes that the input matrix has the form $M = M(:,\mathcal{K}) [I_r, H'] \Pi$
where $\Pi$ is a permutation matrix and $H' \geq 0$ and $||H'(:,j)||_1 \leq 1$ for all~$j$. This means that the columns of $M$ are in the convex hull of the columns of $M(:,\mathcal{K})$; in other words, the columns of
$M(:,\mathcal{K})$ are the vertices of the convex hull of the columns of $M$.
We can identify a vertex of this convex hull using the $\ell_2$ norm as it must be maximized at a vertex. This is the main idea behind SPA which sequentially identifies the columns in $\mathcal{K}$ as follows:
at each step, it first extracts the column of $M$ that has the largest $\ell_2$ norm and then project all columns of $M$ onto the orthogonal complement of the extracted column.
Under the assumption that $M(:,\mathcal{K})$ is full column rank, SPA recovers the set $\mathcal{K}$.

%

Algorithm~\ref{algo:gspa} generalizes SPA in a straightforward manner; we refer to it as generalized SPA (GSPA). At each iteration, it identifies a column or a row of $M$ that will be used as a basis in a GS decomposition.
Each iteration is made of two steps: First, it computes the norms of the columns of $M$ multiplied by $n$ and the norms of the rows of $M$ multiplied by $m$, and selects the column/row corresponding to the largest value.
Second, it projects the columns/rows of $M$ onto the orthogonal complement of the selected column/row.
\begin{algorithm}[h]
\caption{Generalized Successive Projection Algorithm (GSPA) \label{algo:gspa}}
\begin{algorithmic}[1]
\REQUIRE A scaled matrix $M\in\mathbb{R}_{+}^{m\times n}$, number $r$ of columns and rows to extract.
\ENSURE A set $ \mathcal{K}_1\subset \{1,2,\dots,n\}$ of column indices and $\mathcal{K}_2\subset \{1,2,\dots,m\}$ of row indices.
\STATE Let $R=M$, $\mathcal{K}_1=\{\}$, $\mathcal{K}_2=\{\}$.
\WHILE{$R \neq 0$ and $|\mathcal{K}_1|+|\mathcal{K}_2| \leq r$}
\STATE $p=\argmax_{1 \leq j \leq n} n \|R(:,j)\|_2^2$;
\STATE $q=\argmax_{1 \leq i \leq m} m \|R(i,:)\|_2^2$;
\IF{$n \|R(:,p)\|_2^2 \geq m \|R(q,:)\|_2^2 $}
\STATE $R = \left( I-\frac{R(:,p)R^T(:,p)}{\|R(:,p)\|^2_2} \right)R$;
\STATE $\mathcal{K}_1=\mathcal{K}_1\cup \{p\}$;
\ELSE
\STATE $R^T=\left(I-\frac{R(q,:)^TR(q,:)}{\|R(q,:)\|^2_2}\right)R^T$;
\STATE $\mathcal{K}_2=\mathcal{K}_2\cup \{q\}$;
\ENDIF
\ENDWHILE
\end{algorithmic}
\end{algorithm}

One can check that the computational cost of GSPA is $\mathcal{O}(mnr)$ operations; the main operations being matrix-vector products.
As for SPA, GSPA should be applied to a scaled GS matrix.
Unfortunately, as opposed to SPA, there is no guarantee that a column (resp.\@ row) with maximum $\ell_2$ norm will belong to the set $\mathcal{K}_1$ (resp.\@ $\mathcal{K}_2$); see Example~\ref{ex:1} below  where we construct a particular GS matrix for which GSPA fails.
Hence GSPA is a heuristic for GS-NMF.
However, a topic for further research would be to show that GSPA works under suitable additional conditions, that is, for a subset of GS matrices.  In fact, as we will see in the numerical experiments, GSPA works remarkably well for some randomly generated GS matrices.

\begin{example} \label{ex:1}
Let us consider the following (2,2)-separable matrix
\[
M =
\left(
\begin{array}{cccc}
W_1 & W_1 H_1 + W_2 H_2 \\
0 & H_2
\end{array}
\right),
\]\[ \text{with }
W_1 =
\left(
\begin{array}{cc}
 1    & \epsilon \\
 1 & 2 \\
 1 & 3
\end{array}
\right),
H_1 = \left(
\begin{array}{ccc}
 \epsilon  & 2\epsilon & 3\epsilon  \\
  \epsilon    &  1 & 2
\end{array}
\right),
\]
$H_2 = W_1^T$, $W_2 = H_1^T$.
For $\epsilon = 0.001$, we have
\[
M =
\left( \begin{array}{ccccc}
 1 &  0.001 &  0.002 &  0.006 &  0.009 \\
 1 &  2 &  0.006 &  4.004 &  7.005 \\
 1 &  3 &  0.009 &  7.005 &  12.006 \\
 0 &  0 &  1 &  1 &  1 \\
 0 &  0 &  0.001 &  2 &  3 \\
\end{array} \right).
\]
Using SPA, one can check that $M$ is not (4,0)- nor (0,4)-separable. Since there is no pattern of zeros of dimension (1,3) or (3,1), it is not (1,3)- nor (3,1)-separable (see Property~\ref{prop:pattern}).
Therefore, $\mathcal{K}_1 = \{1,2\}$ and $\mathcal{K}_2 = \{5,6\}$ is the only possible GS decomposition with
$|\mathcal{K}_1|+|\mathcal{K}_2| = 4$.
The scaled version of $M$ is
\[
M_s = \left( \begin{array}{ccccc}
 4.654 &  0.028 &  0.251 &  0.034 &  0.033 \\
 0.212 &  2.551 &  0.034 &  1.045 &  1.157 \\
 0.134 &  2.421 &  0.033 &  1.157 &  1.255 \\
 0 &  0 &  4.654 &  0.212 &  0.134 \\
 0 &  0 &  0.028 &  2.551 &  2.421 \\
\end{array} \right).
\]
The column with largest $\ell_2$ norm is the third which is not in $\mathcal{K}_1$, and the row with the largest $\ell_2$ norm is the first which is not in $\mathcal{K}_2$; they both have the same norm.
Therefore, GSPA fails: it returns
$\mathcal{K}_1=\{1,2,3\}$ and $\mathcal{K}_2=\{5\}$,
or $\mathcal{K}_1=\{2\}$ and $\mathcal{K}_2=\{1,4,5\}$ (rows and columns of $M_s$ are the same up to permutations, because $H_2=W_1^T$ and $H_1 = W_2^T$).

Note however that the matrix is almost (3,1)- and (1,3)-separable. In fact,
\[
\frac{\min_{P_1,P_2 \geq 0}
|| M - M(:,1:3) P_1 - P_2 M(5,:)||_F}{||M||_F}
= 0.0244\%.
\]

Note also that the model~\eqref{e3} applied on $M_s$ identifies $X$ and $Y$ perfectly, with the form of~\eqref{eq:XY}.
\end{example}

\section{Numerical Experiments} \label{sec:numexp}

 In this section, we conduct experiments on synthetic (Section~\ref{sec:synth}), document (Section~\ref{sec:doc}) and image data sets (Section~\ref{sec:images}) to test the performance of the proposed  models.
 All experiments were run on Intel(R) Core(TM) i5-5200 CPU @2.20GHZ with 8GB of RAM using Matlab.

Since GS-NMF has not been considered before, we cannot compare GS-FGM and GSPA to existing GS-NMF algorithms.
Instead, we consider a state-of-the-art separable NMF algorithm, namely, the successive projection algorithm (SPA); see the description in Section~\ref{sec:heur}.

Separable NMF algorithms such as SPA can only identify a subset of the columns of the input matrix $M$. Hence, we consider the following three possibilities:
\begin{enumerate}

\item SPA is applied on $M$ to identify $r_1$ important columns of $M$, and then on $M^T$ to identify $r_2$ important rows of $M$. We refer to this variant as SPA*. Note that this is another heuristic to tackle GS-NMF. It is rather different than GSPA that only requires $r$ as an input and identifies automatically the number of columns and rows to extract; see Algorithm~\ref{algo:gspa}.

\item SPA is applied on $M$ to identify $r=r_1+r_2$ columns of $M$. We refer to this variant as  SPA-C.

\item SPA is applied on $M^T$ to identify $r=r_1+r_2$ rows of $M$. We refer to this variant as  SPA-R.

\end{enumerate}

Although the last two approaches (namely, SPA-C and SPA-R) will not be able to tackle GS-NMF, it is interesting to include them in the comparison to see how much GS-NMF algorithms can reduce the approximation error compared to separable NMF algorithms.

\begin{remark}
We have also considered other separable NMF algorithms combined with the above strategies; namely the successive nonnegative projection algorithm (SNPA)~\cite{gillis2014successive},
XRAY~\cite{kumar2013fast} and FGNSR~\cite{gillis2018fast}. However, they provided results similar to SPA hence we do not show these results here for the simplicity of the presentation.
\end{remark}

We have used a stopping criterion for GS-FGM based on the evolution of the iterates and the error:
we stop GS-FGM when one of the following conditions holds:
\[
\frac{e(k) - e(k-1)}{e(k-1)} \leq \delta
\text{ or }
||Z^{(k+1)} - Z^{(k)}||_F
\leq \delta ||Z^{(1)} - Z^{(0)}||_F ,
\]
where $e(k)$ is the objective function at iteration $k$, $Z^{(k)} = (X^{(k)}, Y^{(k)})$
is the solution at iteration $k$, and $0 < \delta < 1$ is a parameter. We will use $\delta = 10^{-4}$ for synthetic data sets and $\delta = 10^{-2}$ for the real data sets (documents and images).

We will also compare these algorithms to the following algorithms:
\begin{itemize}
\item A state-of-the-art NMF algorithm, namely the accelerated hierarchical alternating least squares (A-HALS) algorithm~\cite{gillis2012accelerated}. We will refer to this algorithm as NMF.

\item  A state-of-the-art minimum-volume NMF algorithm~\cite{fu2016robust}. We use the improved implementation that uses a fast gradient method to solve the subproblems in $W$ and $H$ from~\cite{leplat2019minimum}. We will refer to this algorithm as MV-NMF.
\end{itemize}
For both algorithm, we use the default parameters and perform 1000 iterations.

The code is available from \url{https://sites.google.com/site/nicolasgillis/code}.

\subsection{Synthetic data sets} \label{sec:synth}

In this section, we compare the different algorithms on two types of synthetic data sets:
fully randomly generated (Section~\ref{sec:randnoise}),
and the so-called middle-point experiment with
adversarial noise (Section~\ref{sec:advnoise}).

For GS-FGM, we identify the subsets $\mathcal{K}_1$ and $\mathcal{K}_2$ by using the $r_1$ largest diagonal entries of $X$ and $r_2$ largest diagonal entries of $Y$, respectively. In all experiments, we run GS-FGM with the parameter
$\tilde{\lambda} =  0.25$ and
maxiter = 1000.

Given the subsets
$(\mathcal{K}_1,\mathcal{K}_2)$ computed by an algorithm, we will report the following three quality measures:
\begin{enumerate}

  \item The accuracy, defined as
\begin{equation} \label{accuracy}
  \text{accuracy} =
  \frac{|\mathcal{K}^*_1\cap\mathcal{K}_1| + |\mathcal{K}^*_2\cap\mathcal{K}_2|}{|\mathcal{K}^*_1| + |\mathcal{K}^*_2|},
\end{equation}
  where $\mathcal{K}^*_1$ and $\mathcal{K}^*_2$ are the true column and row indices used to generate $M$.
  The accuracy reports the proportion of correctly identified row and column indices.

  Note that the accuracy cannot be computed for NMF and MV-NMF that do not identify columns and rows of the input matrix.

  \item The relative approximation  error, defined as
 \begin{equation} \label{error}
  \frac{\min_{P_1 \geq 0, P_2 \geq0}\|M-M(:,\mathcal{K}_1) P_1 - P_2 M(\mathcal{K}_2,:)\|_F}{\|M\|_F}.
\end{equation}

 Note that we compute $P_1$ and $P_2$ using the coordinate descent method from~\cite{gillis2012accelerated}.

\item The distance to ground truth: given the solution $(W,H)$ of an algorithm, it is
defined as
\begin{equation}\label{distgroundtruth}
  \frac{\min_{\pi_w}
  \|W^*-W(:,\pi_w)\|_F}{2\|W^*\|_F}
  +
  \frac{\min_{\pi_h}
  \|H^*-H(\pi_h,:)\|_F}{2\|H^*\|_F}
 \end{equation}
 where $\pi_w$ and $\pi_h$ are permutations, and $(W^*, H^*)$ is the ground truth that generated the noiseless input data $M = W^* H^*$ (see Definition~\ref{def:gsmat1}).
 Note that for GS-NMF algorithms, $W = [M(:,\mathcal{K}_1), P_2^*]$ and
 $H = [P_1^*; M(\mathcal{K}_2,:)]$ where $P_1^*$ and $P_2^*$ are the solutions of~\eqref{error}.

\end{enumerate}


\subsubsection{Fully randomly generated data} \label{sec:randnoise}

We generate noisy (20,20)-separable matrices $M \in \mathbb{R}^{100\times 100}$ as follows:
$$
\Pi_r \max
\left( 0 ,
 \underbrace{  D_r
\left(
     \begin{array}{cc}
       W_1 & W_1 H_1+W_2 H_2 \\
       0_{20,20} & H_2 \\
     \end{array}
   \right)  D_c
    }_{M^s}
   + N \right) \Pi_c,
$$
where

\begin{itemize}

  \item    The entries of the matrices
  $W_1 \in \mathbb{R}^{80 \times 20}$  and   $H_2 \in \mathbb{R}^{20 \times 80}$ are generated uniformly at random in the interval [0,1] using the \texttt{rand} function of MATLAB. $H_1\in \mathbb{R}^{20 \times 80}$ and  $W_2\in \mathbb{R}^{80 \times 20}$ are generated using sparse uniformly distributed random matrices
  with the density equal to 50\% (\texttt{sprand}(m,n,0.5) in Matlab).

   \item  The diagonal matrices $D_r$ and $D_c$ are computed so that $M^s$ is scaled; we use the algorithm that alternatively scales the columns and rows of the input matrix~\cite{knight2008sinkhorn, olshen2010successive}.

   \item  The entries of the noise $N \in \mathbb{R}^{100 \times 100}$ are generated uniformly at random with the normal distribution of mean 0 and standard deviation 1 using the \texttt{randn} function of MATLAB.  The noise matrix $N$ is then normalized so that
     $||N||_F = \epsilon ||M^s||_F$,  where $M^s$ is the noiseless scaled (20,20)-separable matrix, and $\epsilon$ is a parameter that relates to the noise level.

   \item  $\Pi_r$ and $\Pi_c$ are randomly generated permutation matrices.
\end{itemize}

We use 20 noise levels
$\epsilon$ logarithmically spaced in $[10^{-3},1]$ (in Matlab, \texttt{logspace(-3,0,20)}).
For each noise level, we generate 25 such matrices and report the average quality measures on
Figures~\ref{fullacc}, \ref{fullres} and~\ref{fulldist}.

\begin{figure}
\includegraphics[width=1\textwidth]{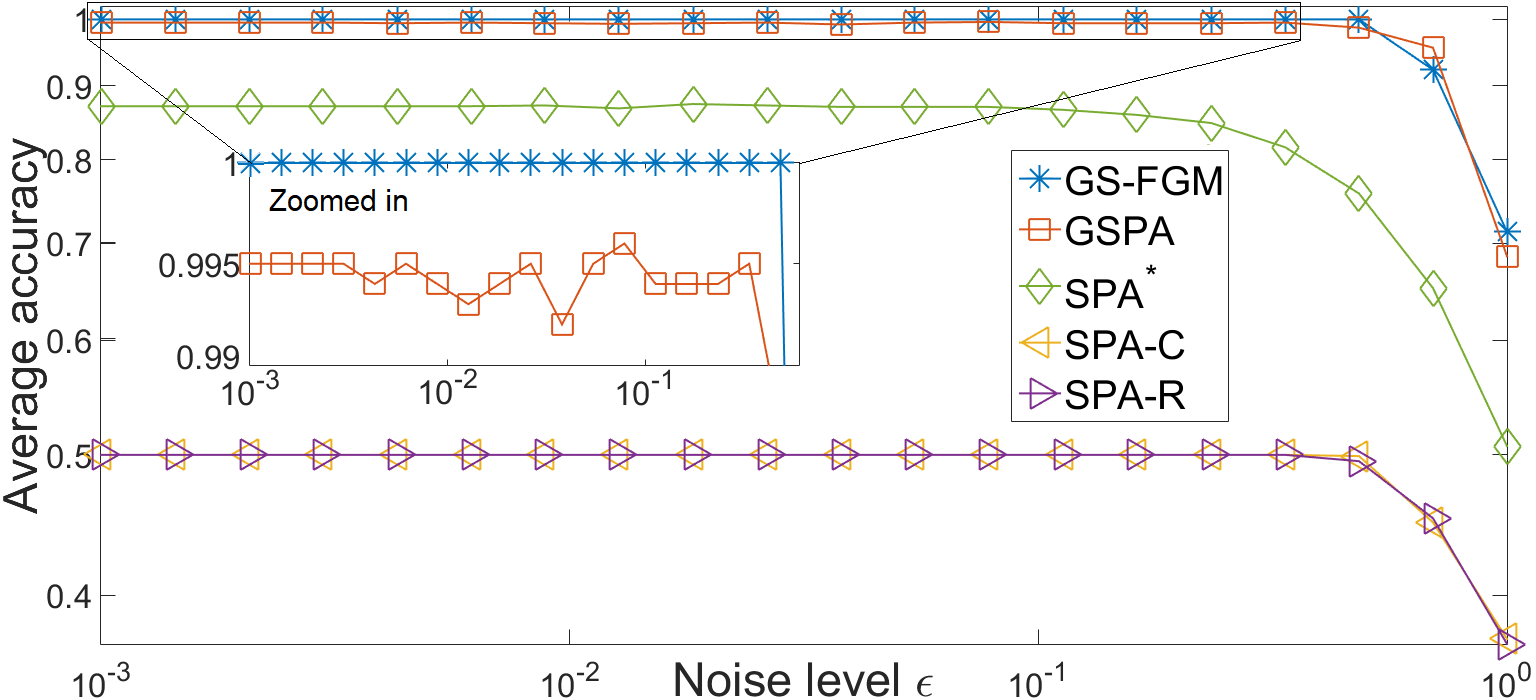}
\caption{Average accuracy~\eqref{accuracy} for the different algorithms on the fully randomly generated GS matrices.  \label{fullacc}}
\end{figure}

\begin{figure}
\includegraphics[width=1\textwidth]{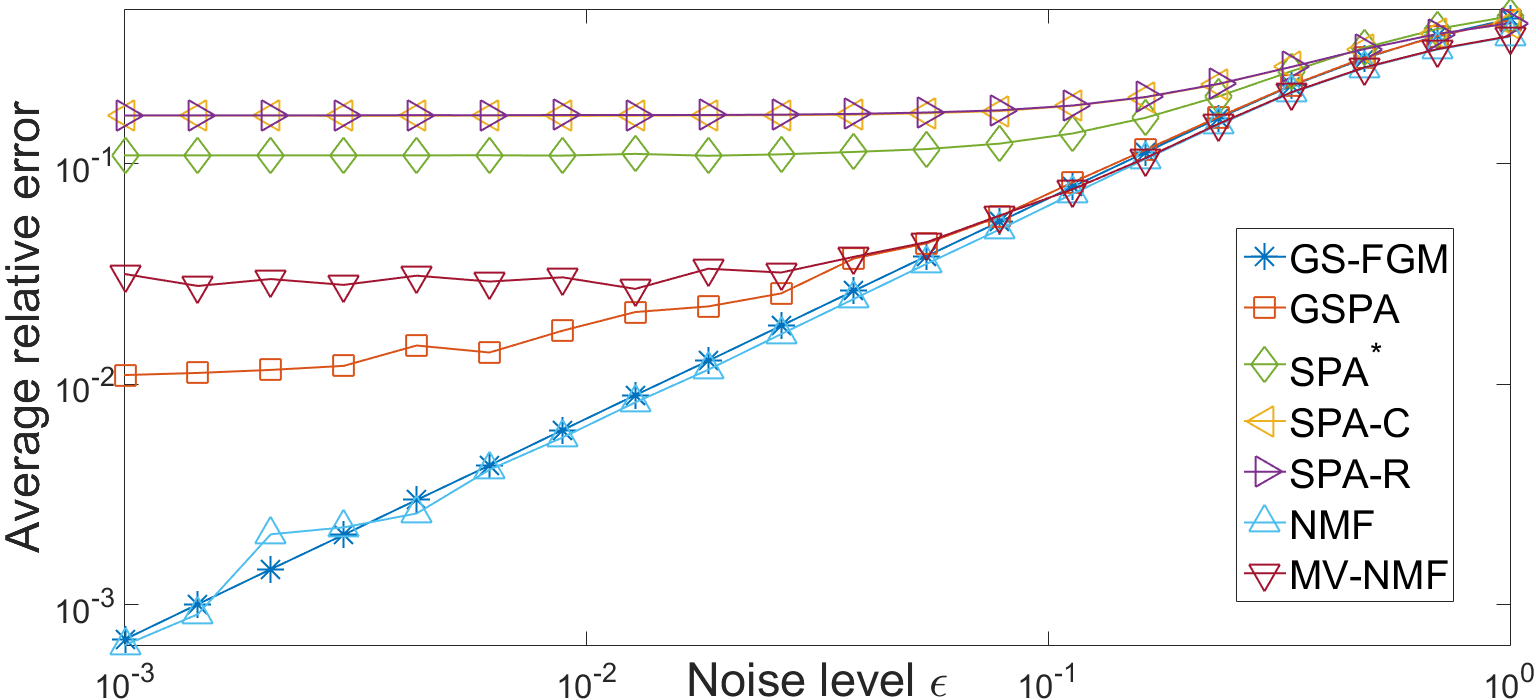}
\caption{Average relative approximation error~\eqref{error} on the fully randomly generated GS matrices.
\label{fullres}}
\end{figure}

\begin{figure}
\includegraphics[width=1\textwidth]{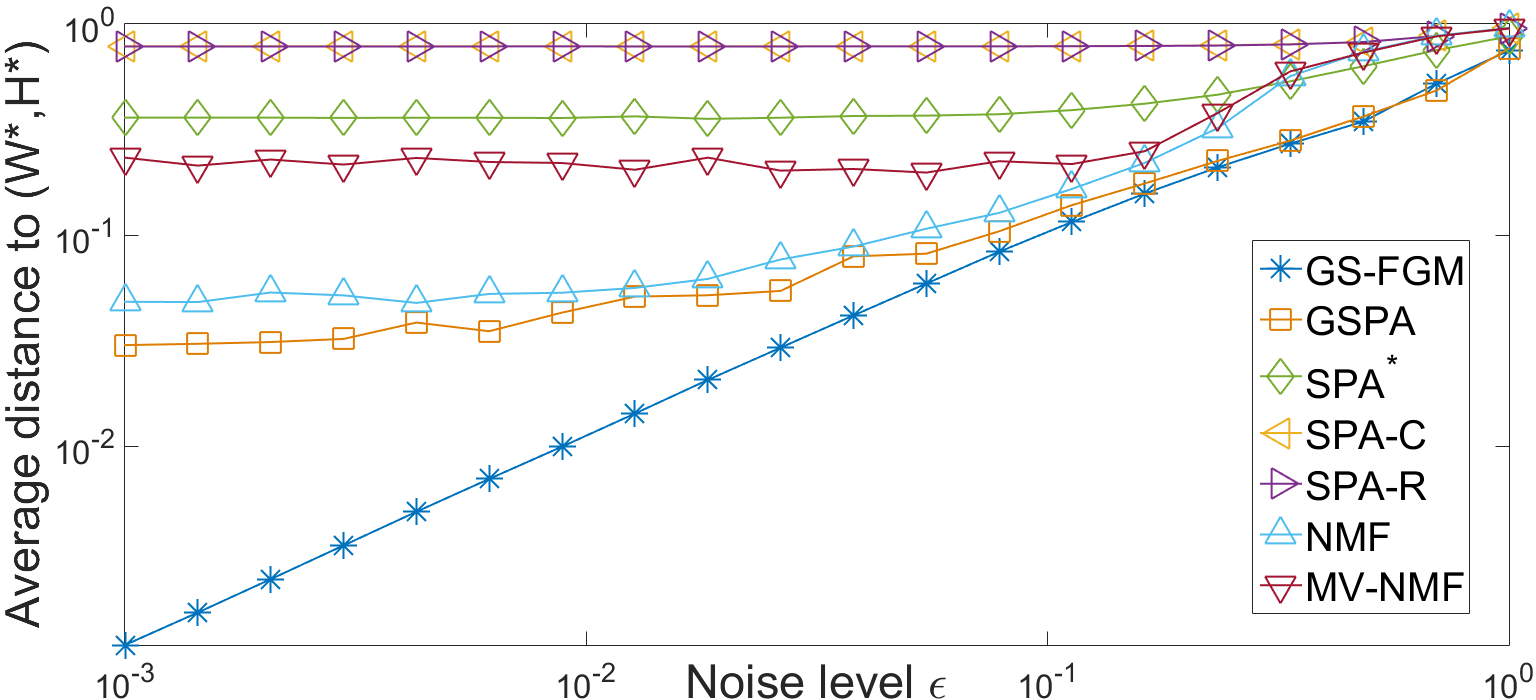}
\caption{Distance to ground truth~\eqref{distgroundtruth} on the fully randomly generated GS matrices.
\label{fulldist}}
\end{figure}

We observe the following:
\begin{itemize}

\item As expected, SPA-C and SPA-R have an accuracy of at most 50\%, and perform very badly to recover the ground truth.
Moreover, they also perform much worse than GS-FGM in terms of approximation error.
This validates the GS-NMF model in the sense that it is able to reduce the approximation error compared to separable NMF for the same factorization rank.

\item  In terms of accuracy, GS-FGM performs the best, having an accuracy of 100\% for all $\epsilon \leq 0.483$. Surprisingly, even for low-noise levels, GSPA is not able to recover exactly all column and row indices (see the zoomed-in graph on Figure~\ref{fullacc}). SPA$^*$ performs better than SPA-C and SPA-R, but much worse than GS-FGM and GSPA.

\item In terms of relative error, NMF performs similarly as GS-FGM. This is not surprising since NMF factorizes the input matrix with no other constraints than nonnegativity. It is actually nice to observe that GS-FGM produces solutions with the same relative error than NMF although this model is much more constrained;
the reason is that the input data satisfies our assumption.

\item In terms of recovering the ground truth,  GS-FGM outperforms all other algorithms, followed by GSPA. NMF and MV-NMF are not able to recover the ground truth due the non-uniqueness of the solution. This shows experimentally the advantage of using GS-NMF to have identifiability of the solution, given that the input matrix is close to being a GS matrix; see the discussion in Section~\ref{sec:uniq}.
\end{itemize}

In the next section, we construct more complicated synthetic data sets for which the behavior of the different algorithms is further highlighted.

\subsubsection{ Middle points and adversarial noise }  \label{sec:advnoise}

In this section, we generate the noisy GS matrices exactly as in the previous section except that
$m = 78$, $n = 55$, $r_1 = 10$, $r_2 = 12$, and
\begin{itemize}

\item   the $\binom{r_1}{2} = 45$ columns $H_1$
(resp.\@ $\binom{r_2}{2}=66$ rows of $W_2$) contain all possible combinations of two non-zero entries equal to $0.5$ at different positions.
Hence, the columns of $W_1H_1$ (resp.\@ rows of $W_2H_2$) are all the middle points of the columns of $W_1$ (resp.\@ rows of $H_2$).

\item    No noise is added to the first $r_1$ columns and last $r_2$ rows of $M^s$, that is, $N(:,1:r_1)=0$ and
$N(m-r_2+1:m,:)=0$, while we set $N(1:m-r_2,r_1+1:n)$ equal to
      $$
       M^s(1:m-r_2,r_1+1:n) -
      \bar{w} e^T -
      e \bar{h},
      $$
      where
      $\bar{w}$ and $\bar{h}$ are the average of the columns of $W_1$ and rows of $H_2$, respectively, that is, $\bar{w} = \frac{1}{r_1} W_1e$ and
       $\bar{h} = \frac{1}{r_2} e^T H_2$.
  Intuitively, the noise will move the data point towards the outside of the convex hull of the columns of $W_1$ and the rows of $H_2$.
  The noise matrix $N$ is normalized so that
  $||N||_F = \epsilon ||M^s||_F$.


\end{itemize}

This example is inspired by the so-called middle point experiment from~\cite{gillis2014fast}. Intuitively, we are moving the data points towards the outside of the set spanned by $W_1$ and $H_2$.

We use the same strategy for the choice of the noise levels, and report the average quality measures over  25 trials on
Figures~\ref{acc-mid}, \ref{res-mid} and~\ref{dist-mid}.
\begin{figure}
\includegraphics[width=1\textwidth]{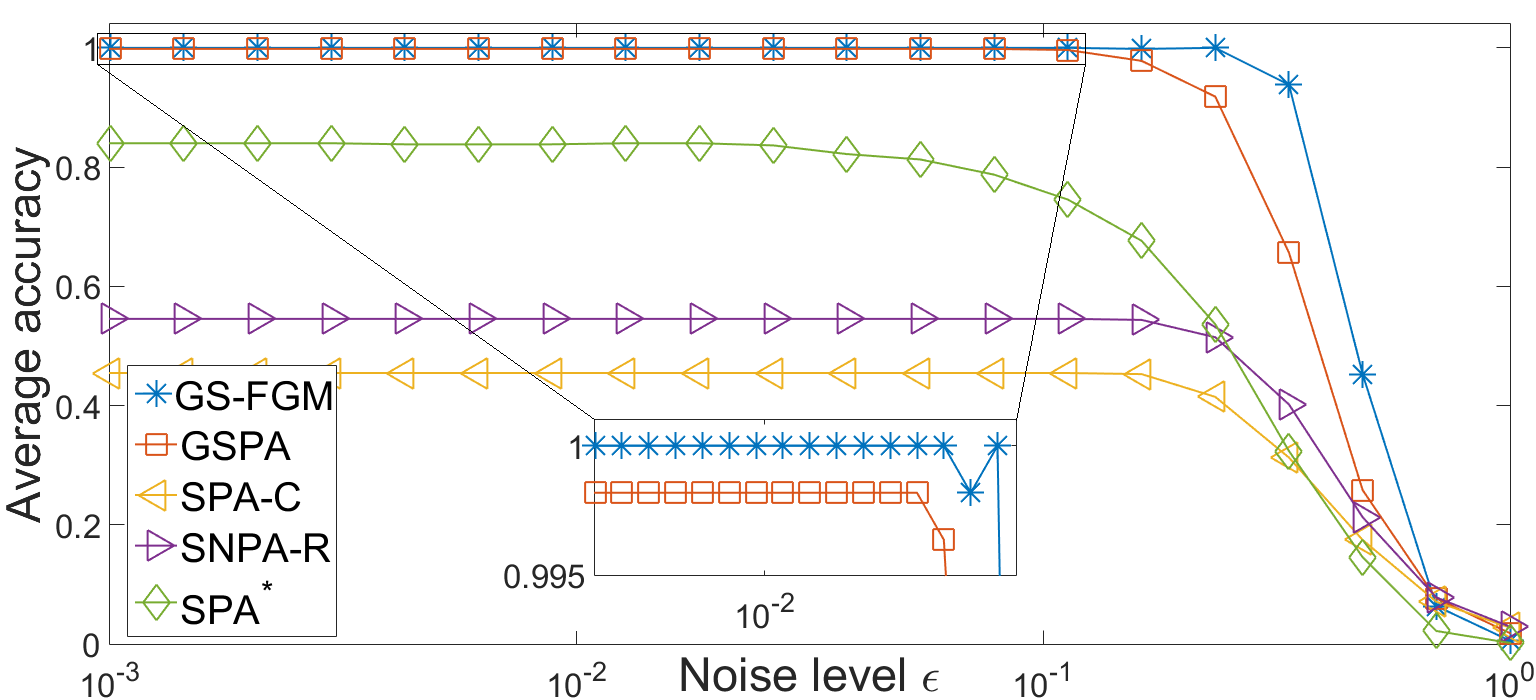}
\caption{Average accuracy~\eqref{accuracy} for the different algorithms on the middle-point GS matrices with adversarial noise. \label{acc-mid}}
\end{figure}
\begin{figure}
\includegraphics[width=1\textwidth]{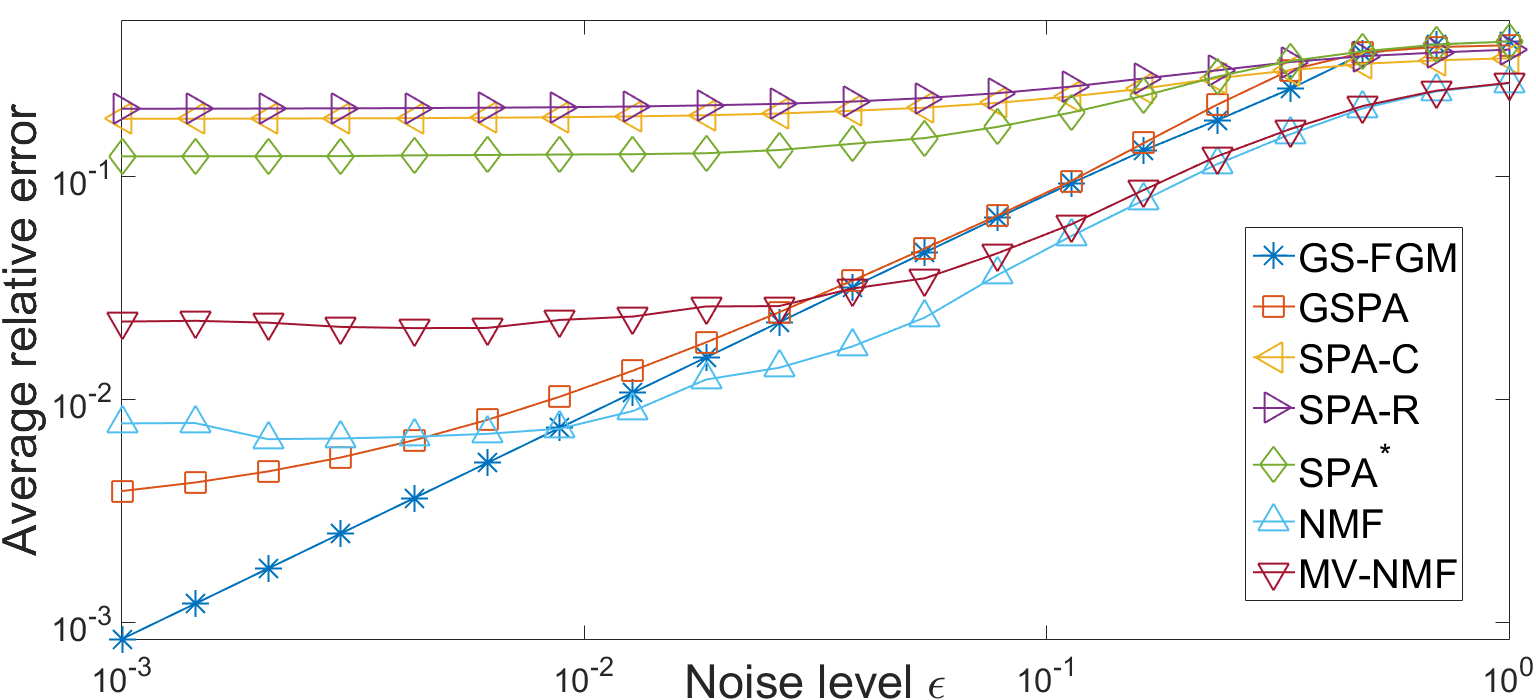}
\caption{Average relative approximation error~\eqref{error} for the different algorithms on the middle-point GS matrices with adversarial noise. \label{res-mid}}
\end{figure}
\begin{figure}
\includegraphics[width=1\textwidth]{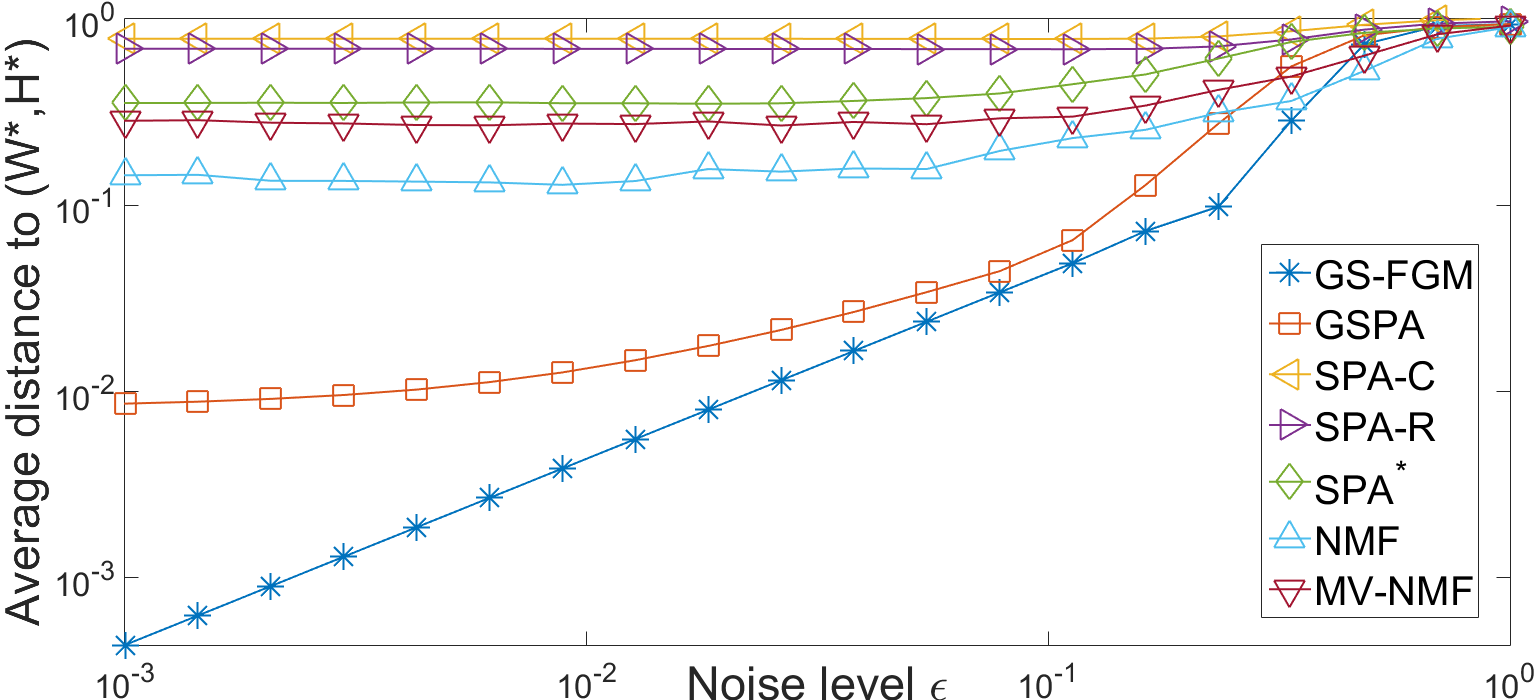}
\caption{Distance to ground truth~\eqref{distgroundtruth} on the middle-point GS matrices with adversarial noise. \label{dist-mid}}
\end{figure}

We observe the following:
 \begin{itemize}
 \item In terms of accuracy, the observations are similar than for the fully random synthetic data sets.  SPA-R and SPA-C are naturally not able to have a good accuracy. Note however that SPA-R performs better than SPA-C because there are more separable rows (12) than columns (10). Moreover, GS-FGM is the only algorithm able to recover the column and row indices perfectly for $\epsilon \leq 0.113$. GSPA performs almost as well but cannot extract all indices (see the zoomed-in graph on Figure~\ref{acc-mid}). SPA$^*$ performs in between.

 \item In terms of approximation error, the behavior is rather interesting: GS-FGM outperforms NMF for low-noise levels ($\epsilon \leq 0.01$) while, for larger noise levels, NMF (and to a lesser extent MV-NMF) performs better.
The reason of the worse performance of NMF is that the problem is more complicated and the NMF algorithm gets stuck in bad local minima. This is a rather interesting observation: \emph{using the GS prior, one can identify better solutions than standard NMF.}

 \item In terms of distance to the ground truth, we oberve a similar behavior as for the fully random synthetic data sets except that NMF and MV-NMF perform even worse because of the more complicated structure of the data.

\end{itemize}

This second experiment shows the superiority of GS-NMF compared to NMF and separable NMF: GS-NMF allows to identify the true underlying factors, leading to low approximation errors.
Among GS-NMF algorithms (namely, GS-FGM, GSPA and SPA*),
GS-FGM performs best producing solutions with higher accuracy, lower approximation error and better identified factors. The second best is GSPA.

\subsection{Document data sets} \label{sec:doc}

In this section, we compare the different algorithms on documents data sets. We use the TDT30 data set~\cite{cai2008modeling},
and the 14 data sets from~\cite{zhong2005generative}.
Note that document data sets are sparse hence are not necessarily scalable hence we did not scale the input matrix.

For GS-FGM, we try 10 different values of $\tilde{\lambda}$ chosen in $[10^{-3},10]$ with 10 log-spaced values (in Matlab, \texttt{logspace(-3,1,10)}), and keep the solution with the highest approximation quality.
The approximation quality is defined as one minus the relative approximation error~\eqref{error}; hence the higher the better.
As opposed to the synthetic data sets,
the numbers $r_1$ and $r_2$ are unknown.
 To evaluate $(r_1,r_2)$ when using GS-FGM,
 we use the strategy described in Section~\ref{sec:fgm} for real data sets.

 \noindent \textbf{Subsampling.} For the document data sets, the size of input data matrix can be very large (the number of words is typically of the order of $10^4$).
 It is impractical to apply GS-FGM such data sets since GS-FGM runs in $\mathcal{O}(mn^2 + nm^2)$ operations.
Similarly as done in~\cite{gillis2018fast}, we preselect a subset of columns and rows of the input matrix. To do so, we adopt the hierarchical clustering from~\cite{gillis2015hierarchical},
running on average in $\mathcal{O}(mn \log_2C)$, where $C$ is the number of the clusters to generate.
For tr11 and tr23 data sets, since the number of documents is relatively small (414 for tr11, 204 for tr23), we keep all the documents and extract 500 words.
For Newsgroups 20, which is a very large data set, we only consider the first 10 classes and refer to the corresponding data set as NG10.
For the other data sets, we extract 500 documents and 500 words, and consider a submatrix matrix $M_s \in \mathbb{R}^{500\times 500}$. However, we take into account the importance of each selected column and row by identifying the the number of data points attached to it (this is given by the hierarchical clustering). To do so, we scale it using the square root of the number of points belonging to its cluster.

 Finally, each algorithm will identify a subset of $r_1$ columns and $r_2$ rows of the subsampled matrix. From these subsets,
 we identify the corresponding columns and rows of the original matrix, and Table~\ref{docus} reports the approximation quality~\eqref{error} of the different algorithms. It also reports the approximation quality of the rank-$r$ truncated SVD, that is,
 $1-\frac{||M-M_r||_F}{||M||_F}$ where $M_r$ is the best rank-$r$ approximation of $M$, to serve as a reference.
 We also only run the separable NMF variants extracting $r_1$ columns and $r_2$ rows using the values of $(r_1,r_2)$ identified by GS-FGM.


\begin{table*}[h]
  \centering
  \resizebox{\textwidth}{30mm}{
  \begin{tabular}{|c|c|cc|ccc|cc||cc|c|}
    \hline
    Dataset&r&$(r_2,r_1)$&GSPA&$(r_2,r_1)$&GS-FGM&SPA*&SPA-C&SPA-R&NMF&MV-NMF&SVD\\
    \hline
    NG10&10&(7,3)&\underline{91.61}&(8,2)&\textbf{91.64}&91.35&91.44&91.49&92.41 &92.37&92.46\\

   TDT30&30&(7,23)&14.13&(4,26)&\underline{14.38}&14.03&\textbf{14.47}&11.30&17.69&17.49& 18.48\\

   classic&4&(4,0)&\textbf{3.58}&(4,0)& \textbf{3.58}&\textbf{3.58}&\textbf{3.58}&1.48&5.12 &3.40&5.20\\
   reviews&5&(0,5)&\textbf{8.39}&(0,5)&\textbf{8.39}&\textbf{8.39}&\textbf{8.39}&7.69& 13.25& 13.06& 13.48\\
   sports&7&(0,7)& \underline{10.49}&(1,6)&\textbf{10.65}&\textbf{10.65}&\underline{10.49}&5.98&13.36  &13.21&13.76\\
    ohscal&10&(0,10)& \textbf{10.27}&(0,10)&\textbf{ 10.27}&\textbf{ 10.27}&\textbf{10.27}&7.03&11.23 &11.14&11.49  \\
    k1b&6& (1,5)&5.76&(1,5)&\textbf{7.07}&5.62&\underline{5.78}&4.54&9.42 &9.16 & 9.62\\
    la12&6&(0,6)&\textbf{4.79}&(0,6)&\textbf{4.79}&\textbf{4.79}&\textbf{4.79}&3.02&7.52 & 5.96&7.78\\
    hitech&6&(3,3)&\textbf{6.43}&(3,3)&\textbf{6.43}&4.50&5.77&4.86&8.84 &8.08& 8.99  \\
    la1&6&(1,5)& 5.05&(2,4) &\textbf{5.13}&4.51&\underline{5.11}&3.73&7.83 &6.50&8.03\\
    la2&6&(0,6)&\textbf{5.86}&(0,6)&\textbf{5.86}&\textbf{5.86}&\textbf{5.86}&3.90&8.11 &7.73&8.35 \\
    tr41&10&(5,5)&52.30&(7,3)&\underline{54.90}&53.31&53.12&\textbf{56.03}&57.19&56.66&57.74\\
    tr45&10&(5,5)&69.08&(7,3)&\textbf{71.94}& 68.20&68.37&\underline{69.55}&76.22 &76.18&76.36\\
    tr11&9&(6,3)& 74.21&(7,2)&\underline{74.27}&72.14&72.50&\textbf{74.74}&76.33&76.28&76.44\\
    tr23&6&(1,5)&63.66&(5,1)&\underline{70.70}&68.46&65.04&\textbf{71.32}& 72.73&72.69&72.86\\
    \hline
    time & & &0.091  & &  1.249  &  0.013  & 0.008  &  0.011  &78.44&280.41&  0.030 \\
    \hline
   \end{tabular}}
  \caption{The relative approximation quality in percent for the document data sets. Among GS-NMF and separable NMF algorithms, the highest quality is highlighted in bold, the second highest is underlined. The last line reports the average computational time in seconds for the different algorithms. \label{docus} }
\end{table*}

We observe the following:

\begin{itemize}
  
\item  GS-FGM and GSPA provide the same solutions in 6 out of the 15 cases. In 5 out of these 6 cases, SPA* provide the same solution.
 
\item  As opposed to the synthetic data sets, SPA-C and SPA-R sometimes perform best, although never significantly better than GS-FGM.
 
\item  GS-FGM performs on average the best, having in all cases the highest or second highest relative approximation quality.

\end{itemize}
  NMF and MV-NMF provide solution with lower approximation error. This is expected since GS-NMF is much more constrained than these NMF variants while the data set is far from being a GS matrix. In fact, we observe that these data sets are not even close to being low rank; see the last column of Table~~\ref{docus} where the relative approximation quality of the truncated SVD is below 10\% for many data sets.
However, it makes sense to perform low-rank approximations to extract meaningful patterns in these documents. In particular, GS-NMF provides subsets of important words and documents; see Table~\ref{docident} for an example.
This illustrates the advantage of interpretability of GS-NMF compared to standard NMF approaches.

The last line of Table~\ref{docus} reports the average computational time in seconds for the different algorithms.
As expected, GS-FGM is slower but the computational time is reasonable for such matrices (below 2.5 seconds in all cases, with an average of 1.25 seconds). Note that NMF and MV-NMF are slower because they are applied directly to the full data sets.

\begin{table}[h]
  \centering
  \begin{tabular}{|l|}
    \hline
    Documents (4) \\ \hline
 \`{}\`{}India won't hesitate to deploy nuclear weapons, premier indicates'' \\
 \`{}\`{}For media, unsavory story tests ideals and stretches limits'' \\
 \`{}\`{}Algeria rebuffs European concerns on reported atrocities''\\
 \`{}\`{}Cohen promises `significant' military campaign against Iraq''  \\  \hline  \hline
Words (26) \\ \hline
hindu,
economic,
pakistan,
iraq,
spkr,
pope,
tobacco,
starr,
suharto, \\
kaczynski,
percent,
white,
school,
jones,
correspondent,
clinton, \\
companies,
jordan,
american,
winter,
lewinsky,
oil,
hong,
hockey,  \\
annan,
president \\
    \hline
   \end{tabular}
  \caption{Words and documents extracted by GS-FGM on the TDT30 data set. \label{docident} }
\end{table}

\subsection{Facial image data sets} \label{sec:images}

In this section, the algorithms are applied on facial image data sets.
In this context, GS-NMF will identify important subjects and important pixels that allow to reconstruct as best as possible the original images.
We use the following facial image data sets:
%

\begin{itemize}
 \item  The CBCL data set is a public database for research usage provided by the MIT center for Biological and Computation Learning. It consists 2429 face images of size $19\times 19$ so that the input pixel-by-face matrix has dimension $361\times 2429$. We set $r=49$ as in~\cite{lee1999learning}.

  \item  The Frey data set is collected by Brendan Frey. It contains 1965 images of Brendan's face and the size of each image is $20\times 28$ so that the input pixel-by-face matrix has dimension $560 \times 1965$. We set $r=50$.

 \item   The Yale data set contains 38 individuals, each of which as 64 frontal face images under different lighting conditions. The images are size of $192\times 168$ which is too large for our purpose (see the discussion in the previous section) hence all the images are downsampled to have size $48 \times 42$. We also select 10 face images from each individual randomly and obtain 380 images. Finally, the pixel-by-face matrix has dimension $2016\times 380$. We set $r=n/10=38$.

 \item  The ORL data set contains a set of faces taken between April 1992 and April 1994 at the Olivetti Research Laboratory in Cambridge, UK. There are ten different images of each of the 40 distinct subjects, each image is size of $112 \times 92$. We subsample each image to obtain images of size $23\times 19$. The pixel-by-face matrix has dimension $437\times 400$. We set $r=n/10=40$.
\end{itemize}

Note that the factorization ranks were chosen rather arbitrarily; we refer the reader to~\cite{tan2012automatic} for a discussion on the choice of $r$.
We use the same strategy to tune $\tilde{\lambda}$ in GS-FGM  as for document data sets. To give each facial image the same importance, we scale them so that their $\ell_1$ norm is equal to one.
The relative approximation quality of the factorizations provided by the different algorithms are reported in Table~\ref{image}.

\begin{table*}[h]
  \centering
  \resizebox{\textwidth}{8mm}{
  \begin{tabular}{|c|c|cc|ccc|cc||cc|c|}
    \hline
    Dataset&r&$(r_2,r_1)$&GSPA&$(r_2,r_1)$&GS-FGM&SPA*&SPA-C&SPA-R&NMF&MV-NMF&SVD\\
    \hline
    CBCL&49 &(1,48) & 80.73 &(14,35)&\underline{83.10}&82.29&79.44&\textbf{84.57} &90.51  & 90.50&91.40 \\
   Frey&50&(24,26) &82.46 &(39,11)&\textbf{83.89}& 83.43&80.61&\underline{83.78}&90.40 &90.41&91.51 \\
   Yale &38&(13,25) &57.52 &(24,14)&\textbf{68.26}& 61.10&60.24&\underline{62.94}& 76.94 &76.85&79.23   \\
 ORL&40&(20,20) &81.38 &(28,12)&\underline{82.54}&82.32&82.23&\textbf{83.26}&  89.47&89.49& 90.24 \\
    \hline
     \end{tabular}}
  \caption{The relative approximation quality in percent for the  facial image data sets. Among GS-NMF and separable NMF algorithms, the highest quality is highlighted in bold, the second highest is underlined. }\label{image}
\end{table*}

For these data sets, GS-FGM outperforms GSPA.
However, SPA-R works very well, slightly better than GS-FGM on CBCL and ORL databases and worse on the Frey and Yale data sets.
The reason is that extracting representative faces within a set of images is not always very appropriate because of the nonnegativity constraints. In some sense, the GS-NMF model is not ideal in this situation, but it is still able to provide meaningful results; for example, for the Yale data sets, it provides significantly lower approximation error than all other algorithms.
Here the non-uniqueness issue plays a role. For example, on the Frey data set, we see that using a (0,49)-separable approximation (SPA-R) leads to an error very close to a $(35,14)$-separable approximation (GS-FGM).

Similarly as for the document data sets, NMF and MV-NMF provide solution with lower approximation error.


Figures~\ref{cbcl:ex}, \ref{frey:ex}, \ref{yale:ex} and
\ref{orl:ex}
provide a visual representations of the solutions generated by GS-FGM: for each data set, they display the positions of the selected pixels and the selected representative faces.
It is interesting to observe the location of the selected pixels: they are either located on the edge (where pixels behave rather differently, not being part of the faces) or are well spread around the center of the face.
The selected faces represent rather different faces from the data sets.
For the CBCL and ORL data sets, the selected faces either come from different persons that look rather different,
or of the same person in very different positions or with different illuminations
(Figures~\ref{cbcl:ex} and~\ref{orl:ex}).
For the Frey data sets, the selected faces represent different emotions (Figure~\ref{frey:ex}).
For the Yale data sets, the selected faces represent different persons and illuminations (Figure~\ref{yale:ex}).

 Figures~\ref{cbcl:rec},
\ref{frey:rec},
 \ref{yale:rec},
 and \ref{orl:rec} display some sample images from the different data sets and their reconstruction using GS-FGM.

\begin{figure}[H]
\includegraphics[width=\textwidth]{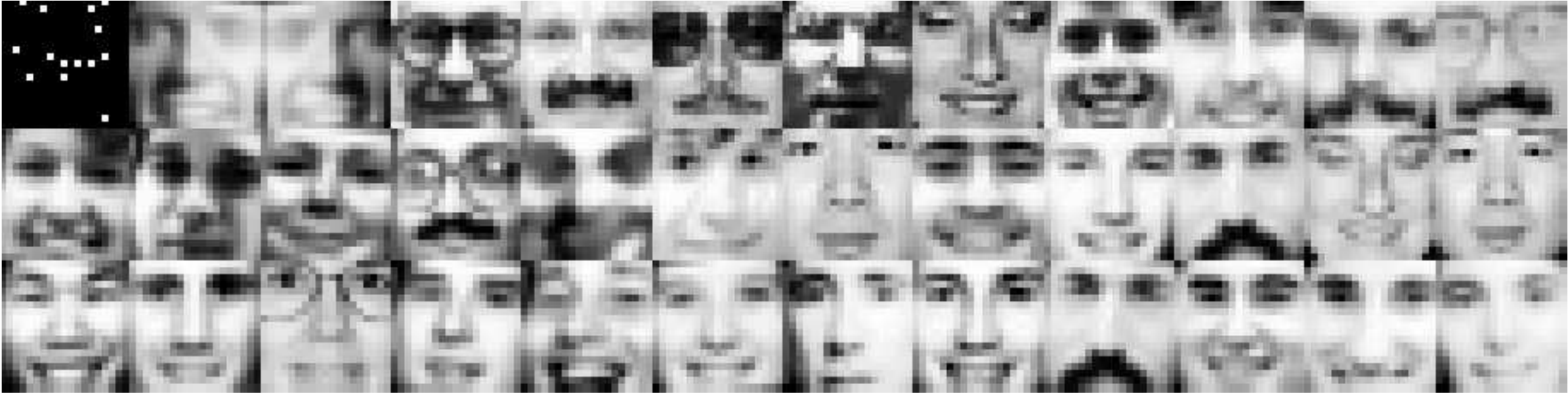} 
 \caption{
 The first image highlights the 14 extracted pixels by GS-FGM for the CBCL data set.
 The next images are the 35 subjects extracted by GS-FGM. \label{cbcl:ex}
 }
\end{figure}


\begin{figure}[H]
\begin{center}
\includegraphics[width=0.75\textwidth]{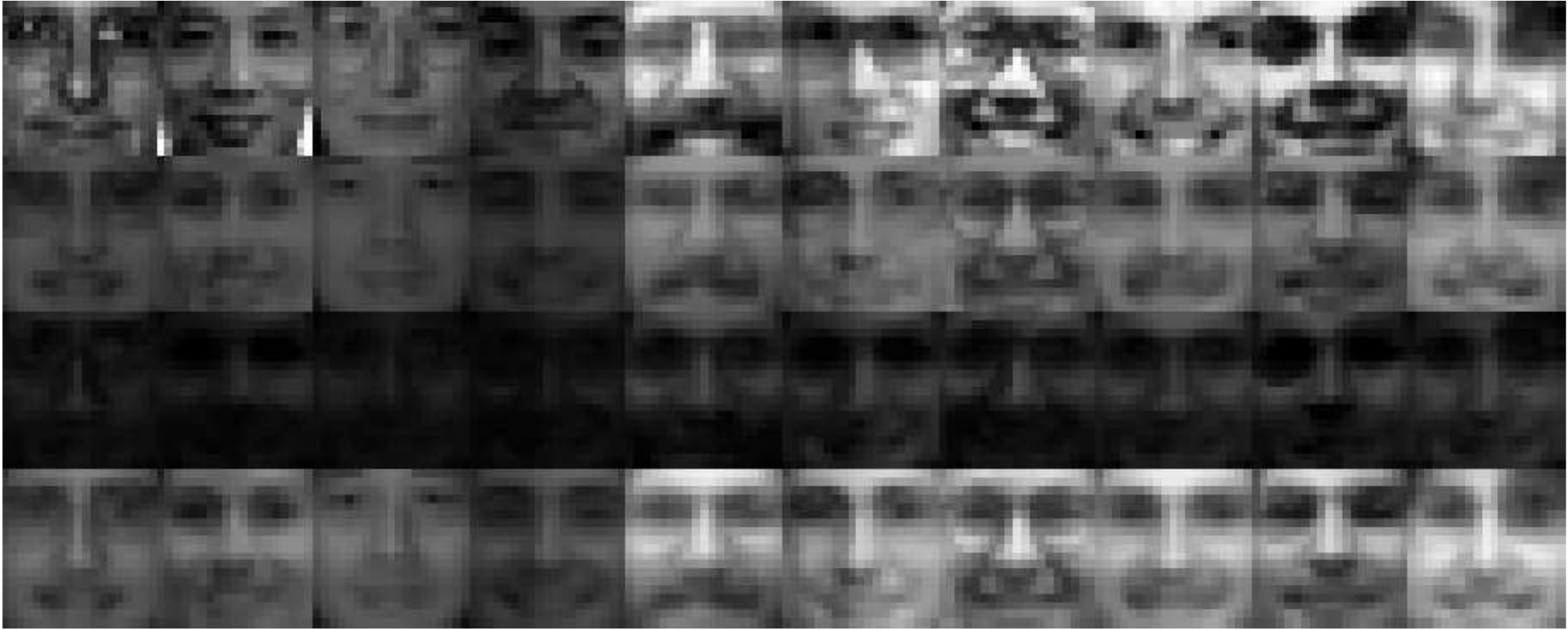}
 \caption{
 The first row displays some images of the CBCL data set,
 the second row displays their approximation by the separable part $M(:,\mathcal{K}_1)P_1$ using the extracted faces,
 the third row displays their approximation by the separable part $P_2 M(\mathcal{K}_2,:)$ using the extracted pixels,
 the last row is the GS approximation $M(:,\mathcal{K}_1)P_1 + P_2 M(\mathcal{K}_2,:)$. \label{cbcl:rec}}
 \end{center}
\end{figure}

\begin{figure}[H]
\includegraphics[width=\textwidth]{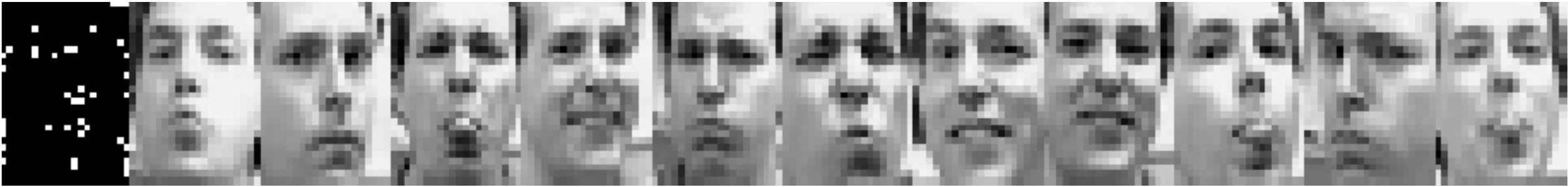}
 \caption{ The first image highlights the 39 extracted pixels by GS-FGM for the Frey data set.
 The next images are the 11 subjects extracted by GS-FGM. \label{frey:ex}}
\end{figure}

\begin{figure}[H]
\begin{center}
\includegraphics[width=0.8\textwidth]{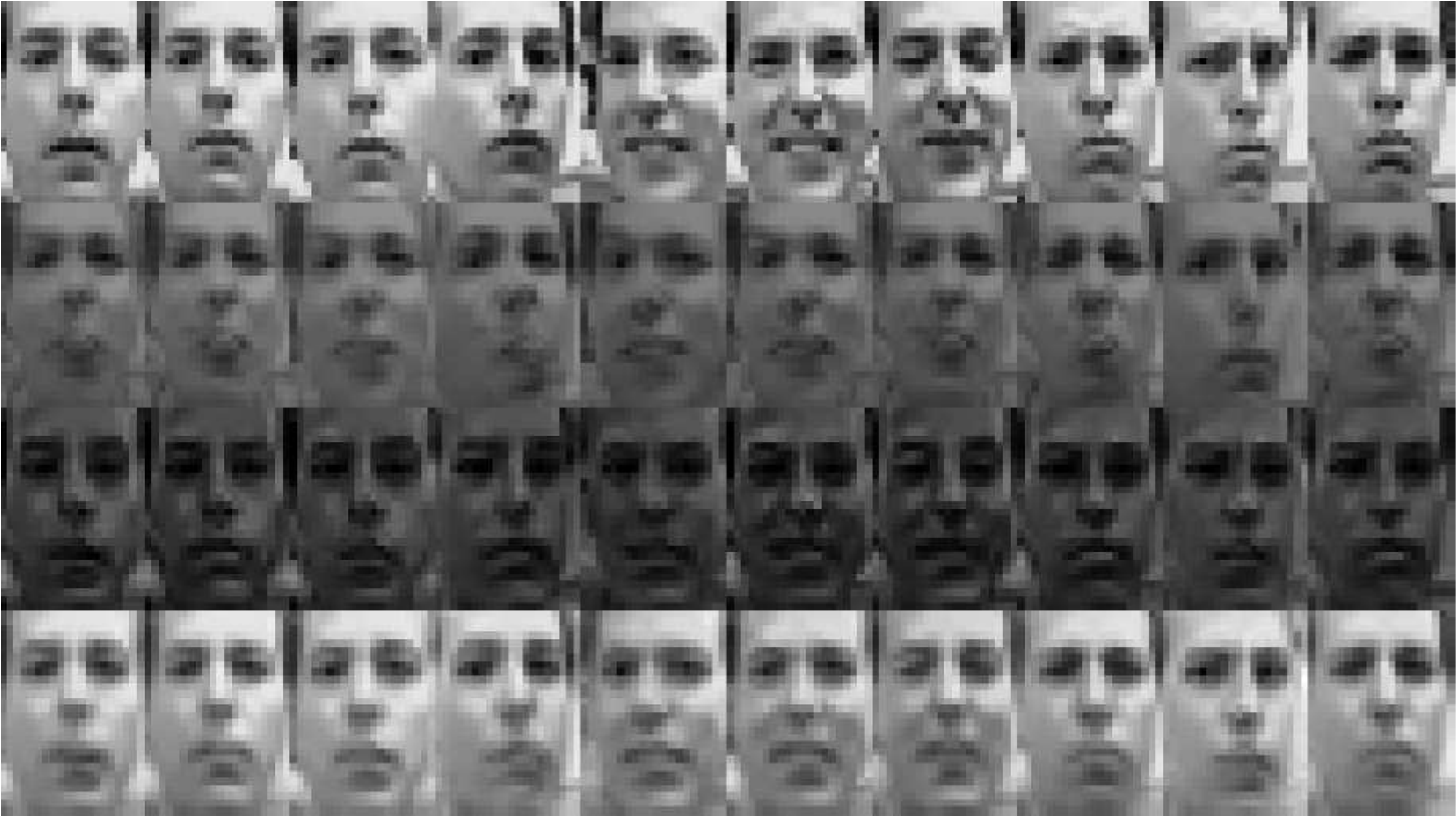}
 \caption{ The first row displays some images of the Frey data set,
 the second row displays their approximation by the separable part $M(:,\mathcal{K}_1)P_1$ using the extracted faces,
 the third row displays their approximation by the separable part $P_2 M(\mathcal{K}_2,:)$ using the extracted pixels,
 the last row is the GS approximation $M(:,\mathcal{K}_1)P_1 + P_2 M(\mathcal{K}_2,:)$. \label{frey:rec}}
 \end{center}
\end{figure}

\begin{figure}[H]
\begin{center}
\includegraphics[width=0.8\textwidth]{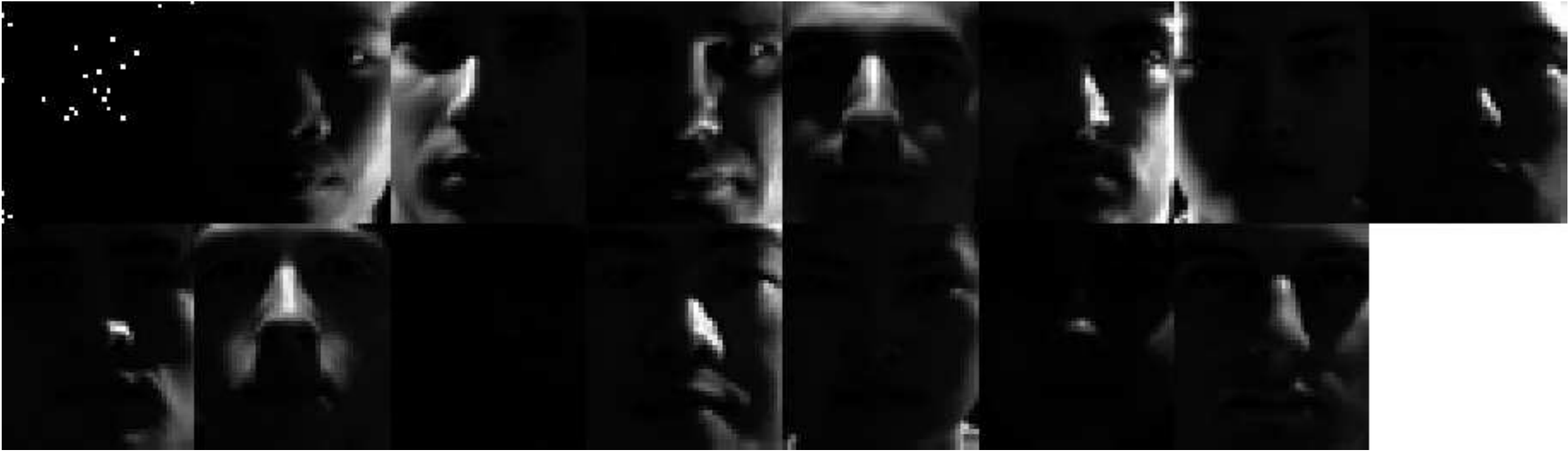}
 \caption{The first image highlights the 24 extracted pixels by GS-FGM for the Yale data set.
 The next images are the 14 subjects extracted by GS-FGM. \label{yale:ex}}
 \end{center}
\end{figure}

\begin{figure}[H]
\includegraphics[width=\textwidth]{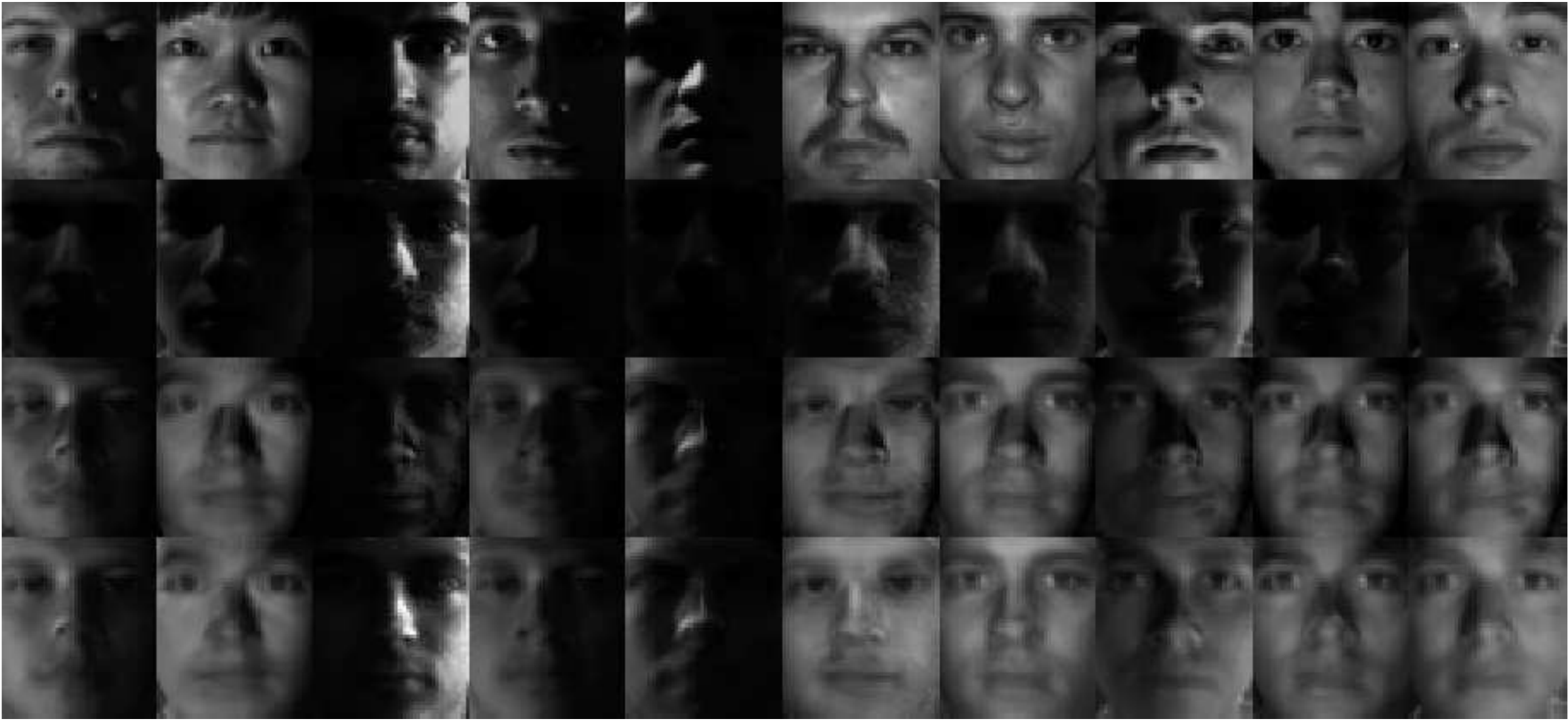}
 \caption{ The first row displays some images of the Yale  data set,
 the second row displays their approximation by the separable part $M(:,\mathcal{K}_1)P_1$ using the extracted faces,
 the third row displays their approximation by the separable part $P_2 M(\mathcal{K}_2,:)$ using the extracted pixels,
 the last row is the GS approximation $M(:,\mathcal{K}_1)P_1 + P_2 M(\mathcal{K}_2,:)$. \label{yale:rec}}
\end{figure}

\begin{figure}[H]
\includegraphics[width=\textwidth]{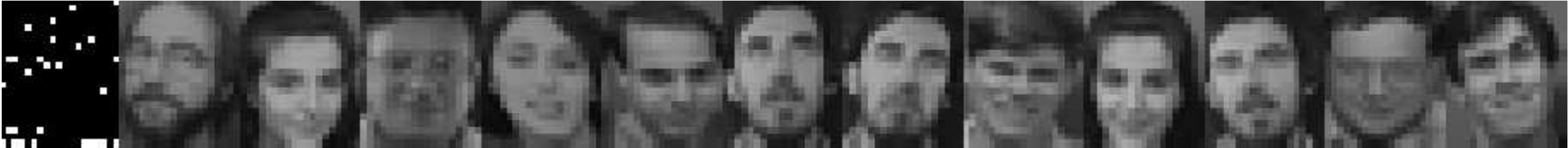}
 \caption{The first image highlights the 28 extracted pixels by GS-FGM for the ORL data set.
 The next images are the 12 subjects extracted by GS-FGM. \label{orl:ex}}
\end{figure}

\begin{figure}[H]
\begin{center}
\includegraphics[width=0.7\textwidth]{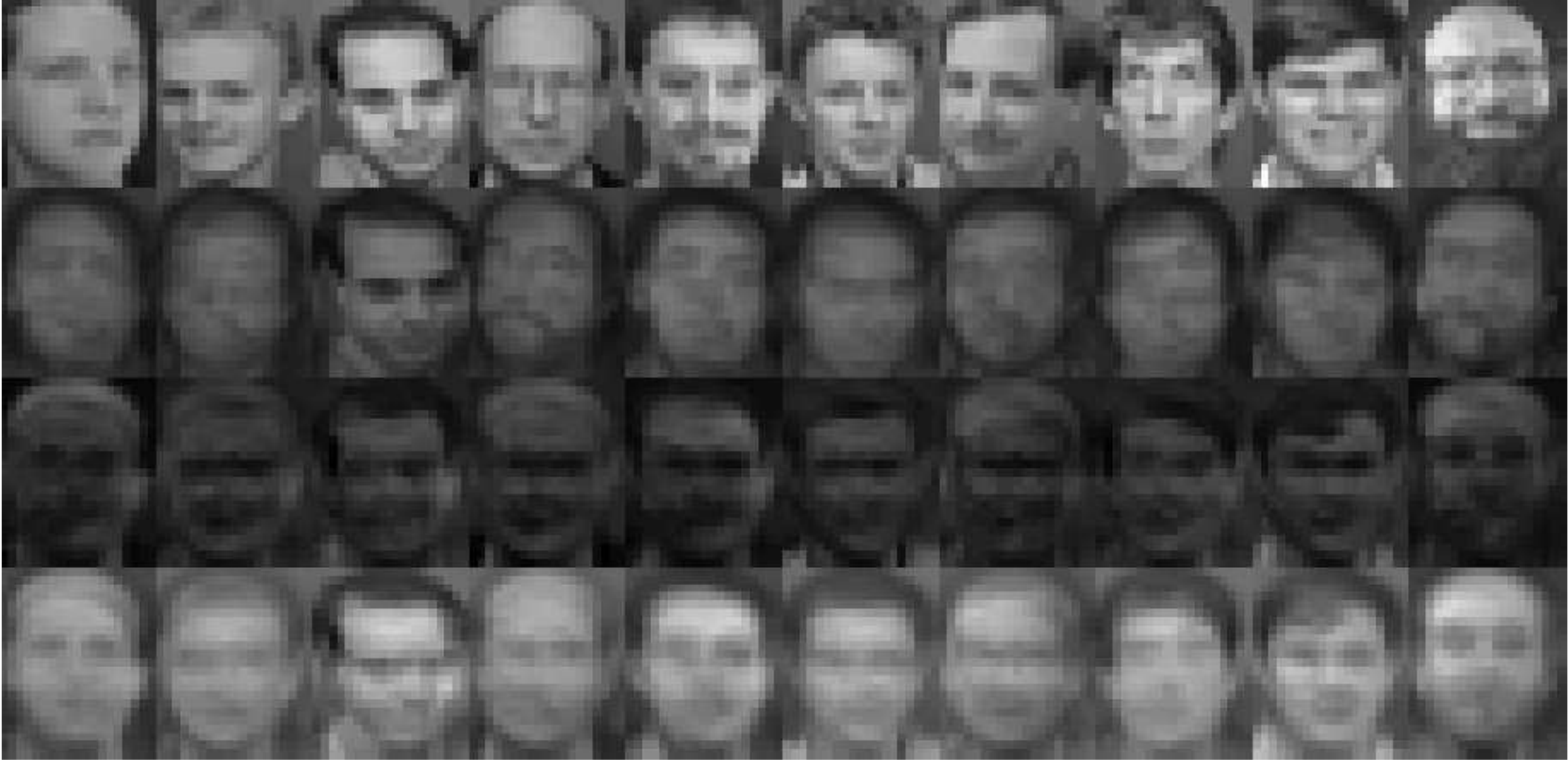}
 \caption{ The first row displays some images of the ORL data set,
 the second row displays their approximation by the separable part $M(:,\mathcal{K}_1)P_1$ using the extracted faces,
 the third row displays their approximation by the separable part $P_2 M(\mathcal{K}_2,:)$ using the extracted pixels,
 the last row is the GS approximation $M(:,\mathcal{K}_1)P_1 + P_2 M(\mathcal{K}_2,:)$. \label{orl:rec}}
 \end{center}
\end{figure}

Table~\ref{timeofimage} reports the computational time for the different algorithms.
As expected, GS-FGM is slower; in particular for the largest data set, namely the Yale data set ($2016 \times 380$), where GS-FGM requires 35 seconds.

\begin{table*}[h]
  \centering
  \resizebox{\textwidth}{8mm}{
  \begin{tabular}{|c|c|cc|ccc|cc||cc|c|}
    \hline
    Dataset&r&$(r_2,r_1)$&GSPA&$(r_2,r_1)$&GS-FGM&SPA*&SPA-C&SPA-R&NMF&MV-NMF&SVD\\
    \hline
    CBCL&49 &(1,48) &1.612 &(14,35)&5.076&0.075&0.070&0.128&17.52  &307.63&0.605 \\
   Frey&50&(24,26) &2.616&(39,11)&5.387&0.096& 0.070&0.090 &  23.94 &392.83&0.603\\
   Yale &38&(13,25) &1.387&(24,14)&35.746&0.053&  0.048&0.044& 12.10  &422.49&0.318 \\
 ORL&40&(20,20) & 0.470&(28,12)&0.905&0.017&  0.013&0.013 &4.62  &159.73&0.051  \\
    \hline
     \end{tabular}}
  \caption{Computational time in seconds for the different algorithms on the image data sets. \label{timeofimage}}
\end{table*}

\subsection{Take-home messages from the numerical experiments}

In terms of approximation error, GS-NMF provides in general results that are better than separable NMF algorithms.
For synthetic data sets, where the input data is close to being a GS matrix, GS-NMF competes favourably with NMF and MV-NMF. In particular, it is able to recover the ground truth factors while standard NMF algorithms fail to do so.
Moreover, for more complicated data sets (see Section~\ref{sec:advnoise}), GS-NMF can even produce solutions with much lower approximation error than NMF whose solutions are stuck at bad local minima.
For real data sets, GS-NMF produces solutions with higher approximation error, because of the strong model assumptions. However, it has the advantage to produce highly
interpretable solutions. The improved interpretability was exemplified on a
document data (see Table~\ref{docident}), and on facial images where GS-NMF identified important pixels and subjects in a set of facial images (see
Figures~\ref{cbcl:ex},
\ref{frey:ex}, \ref{yale:ex} and~\ref{orl:ex}),
which is not possible with any other current NMF algorithm.


\section{Conclusion} \label{sec:concl}

In this paper, we have generalized separable NMF: instead of only selecting columns of the input matrix to approximate it, we allow for columns and rows to be selected. We refer to this problem as generalized separable NMF (GS-NMF).
We studied some interesting properties of matrices that can be decomposed using GS-NMF; they are referred to as GS matrices.
In particular, we showed that GS-NMF can represent matrices much more compactly than separable NMF.
Then, we proposed a convex optimization model to tackle GS-NMF, and developed a fast gradient method to solve the model.
We also proposed a heuristic algorithm inspired by the successive projection algorithm from the separable NMF literature.
We compared the algorithms on synthetic, document and image data sets and showed that they are able, in most cases, to generate decompositions with smaller approximation error than separable NMF algorithms.

Compared to standard NMF algorithms, GS-NMF provides decompositions with higher approximation errors (because of the additional constraints in the decomposition) but provides meaningful and easily interpretable factors.
For example, for facial images, GS-NMF identifies important pixels and subjects in a data set.
Moreover, for synthetic data sets, GS-NMF was able to recover the ground truth factors, sometimes leading to much lower approximation error than NMF algorithms.

Further work include to deepen our understanding of GS matrices. This would hopefully allow for example to design more efficient algorithms that provably recover optimal decompositions under suitable conditions (e.g., uniqueness) and in the presence of noise; as done for separable NMF algorithms.


\small

\bibliographystyle{abbrv}
\bibliography{nmfref}

\end{document}